\RequirePackage{amsthm} 
\RequirePackage{amssymb}
\RequirePackage{algorithm}

\documentclass[11pt]{article} 

\usepackage[margin=1in]{geometry} 
\usepackage{authblk} 
\usepackage{subcaption}
\usepackage[american]{babel}
\usepackage{natbib}
\usepackage{mathtools}
\usepackage{booktabs}
\usepackage{tikz}
\usepackage{algorithmic}
\usepackage{newfloat}
\usepackage{listings}
\usepackage{enumitem}
\usepackage{graphicx}
\usepackage{url}
\usepackage{hyperref} 

\newcommand{\secref}[1]{Section~\ref{#1}}
\newcommand{\figref}[1]{Figure~\ref{#1}}
\newcommand{\lemref}[1]{Lemma~\ref{#1}}

\newcommand{\thmref}[1]{Thm.~\ref{#1}}
\newcommand{\propref}[1]{Proposition~\ref{#1}}
\newcommand{\appref}[1]{Appendix~\ref{#1}}

\newcommand{\ignore}[1]{}
\newcommand{\reals}{\mathbb{R}}

\newcommand{\be}{\mathbf{e}}
\newcommand{\bx}{\mathbf{x}}

\newcommand{\bu}{\mathbf{u}}
\newcommand{\bv}{\mathbf{v}}

\newcommand{\bh}{\mathbf{h}}

\newcommand{\norm}[1]{\|#1\|}
\newcommand{\inner}[1]{\langle#1\rangle}

\newcommand{\pp}{\boldsymbol{p}}

\newcommand{\vv}{\mathbf{v}}

\newtheorem{theorem}{Theorem}[section]
\newtheorem{proposition}[theorem]{Proposition}
\newtheorem{lemma}[theorem]{Lemma}

\DeclareCaptionStyle{ruled}{labelfont=normalfont,labelsep=colon,strut=off} 
\floatstyle{ruled}
\newfloat{listing}{tb}{lst}{}
\floatname{listing}{Listing}

\title{When Can Transformers Count to n?}

\author[1]{Gilad Yehudai}
\author[2,3]{Haim Kaplan}
\author[2]{Guy Dar}
\author[4]{Royi Rassin}
\author[5]{Asma Ghandeharioun}
\author[2,3]{Mor Geva}
\author[2,3]{Amir Globerson}

\affil[1]{NYU Center for Data Science}
\affil[2]{School of Computer Science, Tel Aviv University}
\affil[3]{Google Research}
\affil[4]{Bar-Ilan University}
\affil[5]{Google DeepMind}

\date{} 

\begin{document}
\maketitle

\begin{abstract}
Large language models based on the transformer architecture can solve highly complex tasks, yet their fundamental limitations on simple algorithmic problems remain poorly understood. In this work, we focus on basic counting tasks and investigate how the difficulty of these tasks scales with the transformer embedding dimension, the context length, and the vocabulary size. We reveal a sharp theoretical phase transition governed by the relationship between the embedding dimension and the vocabulary size. When the dimension is at least as large as the vocabulary, transformers can perfectly maintain token counts. However, when the vocabulary exceeds the embedding dimension, the interference between non-orthogonal token representations forces the network weights to scale polynomially. This renders the exact counting algorithm numerically unstable and practically unlearnable. We empirically validate this bottleneck by training transformers from scratch, demonstrating a strict performance drop at the theoretical threshold and catastrophic out of distribution failure when scaling the vocabulary or context length. Furthermore, we show that state-of-the-art pretrained models suffer from similar failure cases. Our work reveals a critical blind spot absent from the current literature regarding the connection among these three parameters, proving that vocabulary size fundamentally dictates the difficulty of counting tasks.
\end{abstract}

\section{Introduction}

Large language models (LLMs) based on the transformer architecture can solve highly complex tasks. Given these successes, a key question arises regarding their fundamental limits and the simple algorithmic tasks they inherently struggle to solve. Empirical probing has highlighted several vulnerabilities, including the needle in a haystack problem \citep{needle,ivgi2023efficient} and length extrapolation failures \citep{levy2024task}. A complementary approach is to use theoretical studies to chart the computational capabilities of transformers \citep{sanford2024representational,wei2022statistically}. In this work, we focus on basic counting tasks. We specifically consider \textbf{Query Count} (QC) and \textbf{Most Frequent Element} (MFE). In the QC task, the model is asked how many times a given query token appears in a sequence. For example, given the sequence \texttt{a a b b a c c d a}, the model needs to count how many times the letter `a' appears. In the MFE task, the model must output the count of the most frequent token overall.

Counting is a natural module on which other algorithms can build. It appears closely related to the needle in a haystack problem where the goal is to locate a specific string in a long context. However, counting is a strictly harder task. The needle in a haystack problem can be solved theoretically by a fixed size transformer regardless of the context length or vocabulary size, since a single attention head can detect and extract a matching token using induction heads \citep{olsson2022context}. Counting requires the model to aggregate information globally and extract a precise integer representation. To understand this difficulty, we must examine three key parameters: the embedding dimension of the transformer $d$, the length of the input context $n$, and the size of the vocabulary $m$. The interplay between these variables introduces a severe and previously under-explored architectural bottleneck.

Counting tasks are challenging for LLMs \citep{barbero2024transformersneedglassesinformation}, and evaluating them out of distribution (OOD) often leads to failure \citep{dziri2023faith}. We wish to understand the theoretical origins of this difficulty and whether it stems from a lack of architectural expressivity or an optimization failure. Our first set of contributions is theoretical. We map the expressivity landscape of transformers on counting tasks and reveal a sharp phase transition based on the relationship between the embedding dimension $d$ and the vocabulary size $m$:

\begin{itemize}[leftmargin=*]
    \item \textbf{The $d \geq m$ regime.} We prove that when the embedding dimension is at least as large as the vocabulary, transformers can easily solve counting tasks using orthogonal token embeddings to construct a perfect context histogram.
    \item \textbf{The $d < m$ regime.} When the vocabulary exceeds the embedding dimension, the token embeddings must be non-orthogonal and geometrically crowded. For the QC task, we show a solution still exists, but we prove it requires an MLP whose width scales linearly with the context length $n$. For the MFE task, we establish a hard communication complexity lower bound and prove that a 1-layer transformer cannot solve the task if $d \ll m$.
    \item \textbf{Exploding weights and Lipschitz bounds.} For general deep transformers in the $d < m$ regime, we prove that resolving the geometric crowding of tokens forces the network Lipschitz constant, and consequently its weights, to scale polynomially with $m$. This implies that while theoretical solutions exist, they become numerically unstable and virtually unlearnable.
\end{itemize}

Our second set of contributions is empirical and designed to directly validate these theoretical bottlenecks. We demonstrate that the theoretical limitations dictate practical model performance across three distinct settings:

\begin{enumerate}[leftmargin=*]
    \item \textbf{Controlled training from scratch.} We train small transformers on counting tasks while varying the embedding dimension, vocabulary size, and context length. We observe a clear phase transition in task accuracy that perfectly aligns with our theoretical $d=m$ threshold. This confirms that the network cannot easily learn to bypass the geometric crowding bottleneck.
    \item \textbf{Out of distribution generalization.} We demonstrate that models trained on small vocabularies or short contexts fail to generalize when evaluated on larger vocabularies or longer contexts. Even when the model perfectly memorizes a counting heuristic for a restricted domain, the underlying algorithm it learns does not extrapolate to the dense $d < m$ regime.
    \item \textbf{Pretrained language models.} We evaluate pretrained models from the Gemini and Gemma families and show they suffer from similar vulnerabilities. When the vocabulary size is artificially restricted to a binary alphabet, performance remains strong. However, as the vocabulary size grows to encompass natural language distributions, counting performance degrades significantly independent of the context length.
\end{enumerate}

Taken together, our theoretical and empirical results highlight a crucial blind spot in the evaluation of long context models. While the community heavily focuses on scaling context length, the vocabulary size fundamentally dictates the difficulty of aggregation tasks. Our findings suggest that relying on standard transformer attention is insufficient for reliable counting, and delegating to external tools like code execution is necessary.

\section{Related Work}
Since the introduction of transformer architectures \citep{vaswani2017attention} and the success of LLMs, there has been much work on evaluating their performance on various tasks \citep[e.g., see][]{srivastava2023beyond}. In particular, much recent work has explored performance on long context tasks, the dependence of accuracy on context length \citep{levy2024task}, and the ability of models to extrapolate to lengths beyond those seen at training \citep{anil2022exploring}. The fact that models often do not perform well on these evaluations has prompted works that try to pinpoint the inherent limitations of transformer models.

Crucially, a common thread in these works is that they primarily focus on how the model scales with context length or network depth of the model, while assuming a small, fixed vocabulary size (e.g., a binary alphabet). This leaves the geometrical and informational bottlenecks caused by the interplay between the embedding dimension and a growing vocabulary size largely unexplored.

A related line of work is to relate transformers to complexity classes. For example, it has been shown that transformers can efficiently approximate Turing machines \citep{wei2022statistically}, and that transformers with bounded precision can only solve problems in uniform $TC^0$ \citep{merrill2023parallelism}. Chain-of-Thought \citep{wei2022chain} has also been analyzed from an expressiveness viewpoint, demonstrating it can substantially improve the expressive power of transformers \citep{feng2024towards}. Our focus is not on the general capabilities of transformers but rather on a specific, seemingly simple problem, and on the ability of transformers to solve it as the vocabulary size grows.

Several recent works have also studied counting in transformers, although they do not account for the connection between counting and the vocabulary size. \citet{behrens2024counting} studies counting in small transformers, providing theoretical upper bounds and identifying relation-based and inventory-based counting strategies. \citet{zhang2024counting} study the effect of tokenization for counting tasks, while \citet{chang2024language,chiang2022overcoming} provide empirical evidence of an inductive bias toward the success of models in counting inductively. From a formal expressivity standpoint, \citet{salzer2025counting} provide a framework for the counting power of transformers, demonstrating their ability to capture complex, semialgebraic counting properties. Furthermore, from a mechanistic interpretability perspective, \citet{golkar2024contextual} investigate transformers on a ``contextual counting'' task, revealing how architectural choices like causal masking affect quantitative localization. Despite these theoretical and empirical advances, the specific geometric bottleneck introduced when the vocabulary size exceeds the embedding dimension remains unaddressed.

\section{Setting and Problem Setup}
\label{sec:problem_setup}

\paragraph{Notation.}
We use bold lowercase letters (e.g., $\bv$) to denote vectors and capital letters (e.g., $M$) to denote matrices. We denote the set of integers $\{1, \dots, N\}$ by $[N]$. The standard basis vectors in $\mathbb{R}^d$ are denoted by $\be_1, \dots, \be_d$. 
\subsection{Transformers}
We consider a transformer architecture taking as input a sequence of $n$ tokens $x_1, \ldots, x_n \in [m]$ from a vocabulary of size $m$. We denote by $X^{(0)} \in \mathbb{R}^{D \times n}$ the matrix where the $i$-th column is the embedding of the $i$-th token, summed with its positional embedding: $X^{(0)}_i = \vv_{x_i} + \pp_i$.

Each layer $\ell \in [L]$ applies a multi-head self-attention mechanism followed by a position-wise feed-forward network (MLP). We denote the input to layer $\ell$ by $X^{(\ell-1)}$. The self-attention mechanism at layer $\ell$ with $H$ heads is parameterized by matrices $W_Q^{(h, \ell)}, W_K^{(h, \ell)} \in \mathbb{R}^{d \times D}$ and $W_V^{(h, \ell)} \in \mathbb{R}^{D \times D}$ for each head $h \in [H]$. 

We first define the attention scores matrix $A^{(h, \ell)} \in \mathbb{R}^{n \times n}$ for head $h$ as:
\begin{equation}
    A^{(h, \ell)} = \text{SM}\left( \frac{(W_K^{(h, \ell)} X^{(\ell-1)})^\top (W_Q^{(h, \ell)} X^{(\ell-1)})}{\sqrt{d}} \right)
\end{equation}
where $\text{SM}(\cdot)$ denotes the row-wise softmax operator. The output of the multi-head attention mechanism, $Z^{(\ell)} \in \mathbb{R}^{D \times n}$, is then computed by aggregating the value heads:
\begin{equation}
    Z^{(\ell)} = \sum_{h=1}^H W_V^{(h, \ell)} X^{(\ell-1)} (A^{(h, \ell)})^\top
\end{equation}
This is followed by a residual connection, such that the intermediate representation is $\tilde{X}^{(\ell)} = Z^{(\ell)} + X^{(\ell-1)}$.

Finally, we apply an MLP, denoted by $\mathcal{N}^{(\ell)}: \mathbb{R}^D \to \mathbb{R}^D$, which operates on each token (column) independently. The MLP consists of two linear transformations with a ReLU activation. The output of the layer is $X^{(\ell)} = \tilde{X}^{(\ell)} + \mathcal{N}^{(\ell)}(\tilde{X}^{(\ell)})$.

We assume that all calculations, including the softmax and intermediate accumulations, are performed with a finite precision of $p$ bits. The final output of the transformer is $X^{(L)}$. Most of our theoretical constructions focus on a simplified architecture consisting of a single layer ($L=1$) and a single attention head ($H=1$). In this regime, we omit the layer and head indices for brevity, and the model dimension satisfies $D=d$. Additionally, our theoretical analysis does not explicitly include normalization layers. We note that this does not limit the generality of our results, as it is possible to implement degenerate normalization layers that act as the identity function (see, e.g., \citet{sanford2024representational}). In contrast, our empirical evaluations utilize standard transformer layers that include normalization.

\subsection{Counting Tasks}
We analyze the expressiveness of transformers on two fundamental counting tasks. In both cases, the input is a sequence $x_1, \ldots, x_n$, and the model must output a single scalar value $y$.

\textbf{Query Count (QC).}
The model is presented with a sequence and must return the count of the last token (the query) within that sequence. Formally, the target output is:
\begin{equation}
    y_{QC} = \sum_{i=1}^n \mathbb{I}[x_i = x_n]
\end{equation}
Note that $y_{QC} \geq 1$ always, as the query token itself is part of the sequence.

\textbf{Most Frequent Element (MFE).}
The model must identify the token that appears most frequently in the sequence and output its count. Formally, let $c_z = \sum_{i=1}^n \mathbb{I}[x_i = z]$ be the count of token $z$. The target output is:
\begin{equation}
    y_{MFE} = \max_{z \in \Sigma} c_z~,
\end{equation}
where $\Sigma$ is the vocabulary. For simplicity in our theoretical analysis, we consider the output to be the count of the most frequent element rather than the token itself.

\section{Query Count}
\label{sec:QC}
In this section, we analyze the Query Count (QC) task and ask which transformer architectures can implement it successfully. We distinguish between two main regimes based on the relationship between the model dimension $d$ and the vocabulary size $m$. First, we describe a $\operatorname{Histogram}$ solution that is efficient but requires $d \geq m$. Second, we analyze the $\operatorname{CountAttend}$ mechanism, which functions for any $d  = \Omega(\log(m))$ but essentially requires the MLP width to scale with the context length $n$. Finally, we provide a general impossibility result showing that for a fixed dimension $d$, the error in counting must grow as the vocabulary size $m$ and the context size $n$ increase, regardless of the specific construction used.

\subsection{The Histogram Solution and the $d < m$ Threshold}
\label{sec:hist_solution}

When the embedding dimension allows for orthogonal representations of the vocabulary (specifically $d \geq m$), the QC task can be solved efficiently using a "Histogram" approach.

\begin{theorem}[\textbf{Histogram Solution}]\label{thm:histo_exists}
    For the Query Count problem and any context length $n$, if $d \geq m$, there exists a 1-layer, 1-head transformer with an MLP of width $d$ that solves the task perfectly.
\end{theorem}

The construction relies on mapping each token to a standard basis vector $\be_i$. A single attention head with uniform attention ($Q=0, K=0$) aggregates these vectors to form a histogram of the context: $\bh = \sum_{j=1}^n \be_{x_j}$. The query token $x_n$ is preserved via the residual connection (it can also be preserved by adding a second head). An MLP can then extract the count corresponding to the query index from $\bh$.

\textbf{Failure at $\mathbf{d < m}$.} 
This solution breaks fundamentally when the vocabulary exceeds the dimension. If $d < m$, the embeddings cannot be orthogonal. The interference between non-orthogonal vectors introduces noise into the histogram. Using Welch bounds \citep{welch1974lower}, we show that this interference is unavoidable:
\begin{theorem}\label{thm:histo_breaks}
    If $m \geq 2d$, for any choice of embedding vectors, there exist inputs where the Histogram solution incurs an error of at least $\Omega(\sqrt{n})$.
\end{theorem}
This implies that simple linear aggregation capability of transformers is insufficient for counting in the "dense" vocabulary regime.

\subsection{The CountAttend Solution and Scaling Limits}
\label{sec:count_attend}

When $d < m$, we cannot rely on orthogonal embeddings. An alternative strategy, which we term $\operatorname{CountAttend}$, utilizes the self-attention mechanism to filter tokens dynamically. This solution works for any vocabulary size and dimension, provided the MLP is sufficiently large.

\begin{proposition}\label{prop:softmax_construction}
    For any $n$ and  $d > \log(m)$, there exists a transformer that solves the QC problem with one layer, one attention head, embedding dimension $d$, and an MLP of width $O(n)$.
\end{proposition}

The mechanism works by letting the query token $x_n$ attend to the sequence with specific query/key matrices that assign high attention weights only to tokens identical to itself. This results in an aggregated value vector that scales with the count of the query token. The MLP then acts as a functional inverse, mapping this aggregate value back to the precise integer count. Crucially, because the attention mechanism applies a softmax normalization, recovering the raw count requires inverting the function $f(x) \propto 1/x$. Implementing this inverse with high precision over the domain $[1/n, 1]$ requires the MLP capacity to grow linearly with the context length $n$. For more details and intuition about the construction, see    \appref{app:count_attend_details}.

\textbf{Failure at Scale.} 
While $\operatorname{CountAttend}$ is theoretically powerful, its reliance on an $O(n)$ width MLP makes it impractical for long contexts.
\begin{lemma}\label{lem:1/x_approx}
    Any MLP with ReLU activations that approximates $f(x) = 1/x$ for $x \in [1/n, 1]$ with sufficient precision to distinguish counts requires width $\Omega(n)$.
\end{lemma}
This result implies that a fixed-size transformer cannot count accurately for arbitrarily long sequences using this mechanism. Furthermore, the construction requires attention logits to scale with $\log n$, which introduces optimization difficulties and precision issues in practice.

\subsection{General Limitation for Fixed Dimension}
\label{sec:general_limitation}

The limitations of the Histogram and CountAttend solutions are not coincidental. We now present a general result showing that transformers with fixed embedding dimension $d$ fundamentally struggle to count as the vocabulary size $m$ grows large. This result assumes the model utilizes the Lipschitz continuity typical of bounded-weight neural networks.

\begin{theorem}\label{thm:general_impossibility}
    Let $T$ be an $L$-layer transformer with embedding dimension $d$. Assume the input embeddings are bounded in norm by $R$, the value matrices are bounded such that $\|W_V^{(h,\ell)}\|_2 \leq K_1$ for any layer $\ell\in[L]$ and attention head $h\in[H]$, and also that the MLPs are $K_2$-Lipschitz. Then, there exists a sequence of length $N+1$ for which the counting error is lower bounded by:
    \begin{equation}
        |T(\bx) - y_{QC}| \geq \frac{N}{2} - \frac{(HK_1 K_2)^L \cdot R}{m^{1/d} - 1}
    \end{equation}
\end{theorem}

This theorem (proven in \appref{app:proofs_QC}) reveals a "curse of dimensionality" for counting, and shows that some components of the transformer must scale with both the vocabulary size and context length. As the vocabulary size $m$ increases for a fixed $d$, the term $m^{1/d}$ approaches 1, causing the denominator to vanish and the error bound to dominate. This implies that to maintain low error, the dimension $d$ must scale logarithmically with $m$. 
Moreover, even if $d = \Omega(\log(n))$, to achieve a small error the Lipschitz parameters of the model must scale with the context size $n$. This is similar to the limitation of the $\operatorname{CountAttend}$ solution which essentially requires scaling the Lipschitz constant of the MLP by $n$.

\section{Most Frequent Element}
\label{sec:most_frequent}

We now turn to the Most Frequent Element (MFE) task. In this setting, the model is presented with a sequence of tokens and must identify the count of the token that appears most frequently. This problem generalizes Query Counting, as the model must implicitly track counts for all tokens simultaneously to perform the maximization.

\subsection{Information Theoretic Lower Bound ($L=1$)}
We begin by establishing a lower bound for a single-layer transformer using communication complexity. This result is strong because it provides explicit capacity bounds on the architecture parameters ($d, h, p$) regardless of the weight magnitudes.

\begin{theorem}\label{thm:MFE_lower_bound_comm}
    Consider a 1-layer transformer with $h$ heads, embedding dimension $d$, and precision $p$, followed by an MLP of arbitrary width and depth. If this model solves the MFE task for sequences of length $n$ and vocabulary size $m$, then it must hold that:
    \[ dhp \geq \Omega(\min\{m, n\}) \]
\end{theorem}
The intuition behind this result relies on a reduction from the Set Disjointness problem. We construct a sequence where the prefix encodes a set $A$ (Alice) and the suffix encodes a set $B$ (Bob), using unique tokens for each element. If $A$ and $B$ share an element, that corresponding token appears twice (once in the prefix, once in the suffix), whereas if they are disjoint, every token appears at most once. Consequently, distinguishing between a maximum frequency of $1$ and $2$ is equivalent to determining disjointness. Since the attention mechanism acts as the sole communication channel transferring information from the prefix to the suffix, its total capacity $O(d \cdot h \cdot p)$ must meet the $\Omega(n)$ communication requirement of the problem. The formal proof can be found in \appref{app:proofs_MFE}.

\paragraph{Limitations and Extensions.}
The strength of \thmref{thm:MFE_lower_bound_comm} lies in its explicit capacity bounds for shallow models. However, this argument does not trivially extend to deeper transformers (i.e., $L > 1$). In a multi-layer setting, the "state" passed between layers consists of all $n$ token embeddings, effectively providing a channel capacity of $n \times d$. This allows intermediate layers to distribute and mix information across the entire sequence, bypassing the single-vector bottleneck of the first layer.
Consequently, to prove limitations for deep transformers, we cannot rely on information capacity alone. Instead, we bound the sensitivity of the network.
In the next subsection, we show that regardless of depth or capacity, if $d < m$, the input embeddings become crowded, forcing the network's weights to explode to distinguish them.

\subsection{Weight Magnitude Lower Bound ($L \geq 1$)}
When the embedding dimension $d$ is smaller than the vocabulary size $m$, the token embeddings are necessarily crowded in $\mathbb{R}^d$. We show that differentiating between these crowded vectors requires the network's Lipschitz constant (and consequently its weights) to grow with $m$.

\begin{theorem}\label{thm:MFE_lower_bound_weights}
    Let $T$ be an $L$-layer transformer with embedding dimension $d$. Assume the input embeddings are bounded in norm by $R$. Let $K$ be the Lipschitz constant of the transformer with respect to the inputs. 
    If $d < m$, then to solve the MFE task, $K$ must satisfy:
    \[ K \ge \Omega\left( \frac{\sqrt{n}}{R} \cdot m^{1/d} \right) \]

\end{theorem}

The formal proof can be found in \appref{app:proofs_MFE}. Note that, unlike \thmref{thm:general_impossibility}, here we rely on the geometric properties of the self-attention mechanism itself. While the softmax function can approximate a discontinuous hardmax (implying an infinite Lipschitz constant), doing so for inputs that are $\epsilon$-close requires the pre-softmax logits to differ by a constant amount. Thus, resolving the crowding of the embedding space forces the weight magnitudes to grow with $m$ and $n$.

\subsection{Upper Bound: $d=m$ is Sufficient}
The lower bounds suggest that MFE is difficult when $d \ll m$. We show that the condition $d=m$ is sufficient to solve the task with well-behaved weights.

\begin{theorem}\label{thm:MFE_upper_bound}
    There exists a 1-layer transformer with $d=m$, $H=1$, and logarithmic precision $p = \Omega(\log (n))$ that solves the MFE task. This construction requires an MLP of width $O(m^2)$.
\end{theorem}

The construction mirrors the Histogram solution from the QC task. Since $d=m$, we can orthogonally embed the tokens and aggregate them to form a precise histogram. Because the vectors are orthogonal, standard attention mechanisms with small weights suffice to process them. The full proof can be found in \appref{app:proofs_MFE}.

One limitation of the 1-layer construction is that extracting the maximum from the $m$-dimensional histogram requires an MLP of width $O(m^2)$. This wide MLP can be avoided by trading width for depth. A natural alternative is to use a deeper MLP: a depth $O(\log (m))$ MLP with width $O(m)$ can calculate the maximum efficiently via a binary tree of pairwise comparisons \citep{safran2024many}. Another option is to use a 2-layer transformer architecture. In this setup, the first layer calculates the count for each element, and the second layer utilizes its softmax attention as a hardmax operator to directly attend to and extract the maximal count. We present the above construction for simplicity.

\section{Experiments}
\label{sec:experiments}

Our theoretical analysis highlights a fundamental bottleneck in the transformer architecture's ability to count, driven by the relationship between the embedding dimension $d$, the vocabulary size $m$, and the sequence length $n$. In this section, we empirically validate these limitations on the Query Counting (QC) task. The model is given a sequence of length $n$ containing integer tokens sampled uniformly from a vocabulary of size $m$, and must output the exact frequency of every possible token. All the experimental details can be found in \appref{app:experimental_details}. 

\paragraph{Evaluation Metric: NMAE.} 
Standard accuracy is too rigid for this task, as it treats all incorrect counts equally and obscures near-miss learning signals. Conversely, standard Mean Absolute Error (MAE) is heavily skewed by the vocabulary size. If the context lenght $n$ is fixed and $m$ grows, the expected frequency of any given token is $\frac{n}{m}$, which shrinks towards zero. This artificially deflates the absolute error. To establish a fair comparison across arbitrary scales, we evaluate using Normalized Mean Absolute Error (NMAE). By dividing the absolute error by the expected frequency, NMAE normalizes the difficulty. An NMAE of 0.1 indicates the model's prediction deviates by 10\% of the expected count, allowing us to accurately compare model performance across widely varying $m$ and $n$ dimensions. Each experiment is repeated $5$ times with random seeds, and both the mean and variance are reported.

\subsection{In Distribution Scaling}

\begin{figure*}[t]
    \centering
    \begin{subfigure}[b]{0.32\textwidth}
        \centering
        \includegraphics[width=\textwidth]{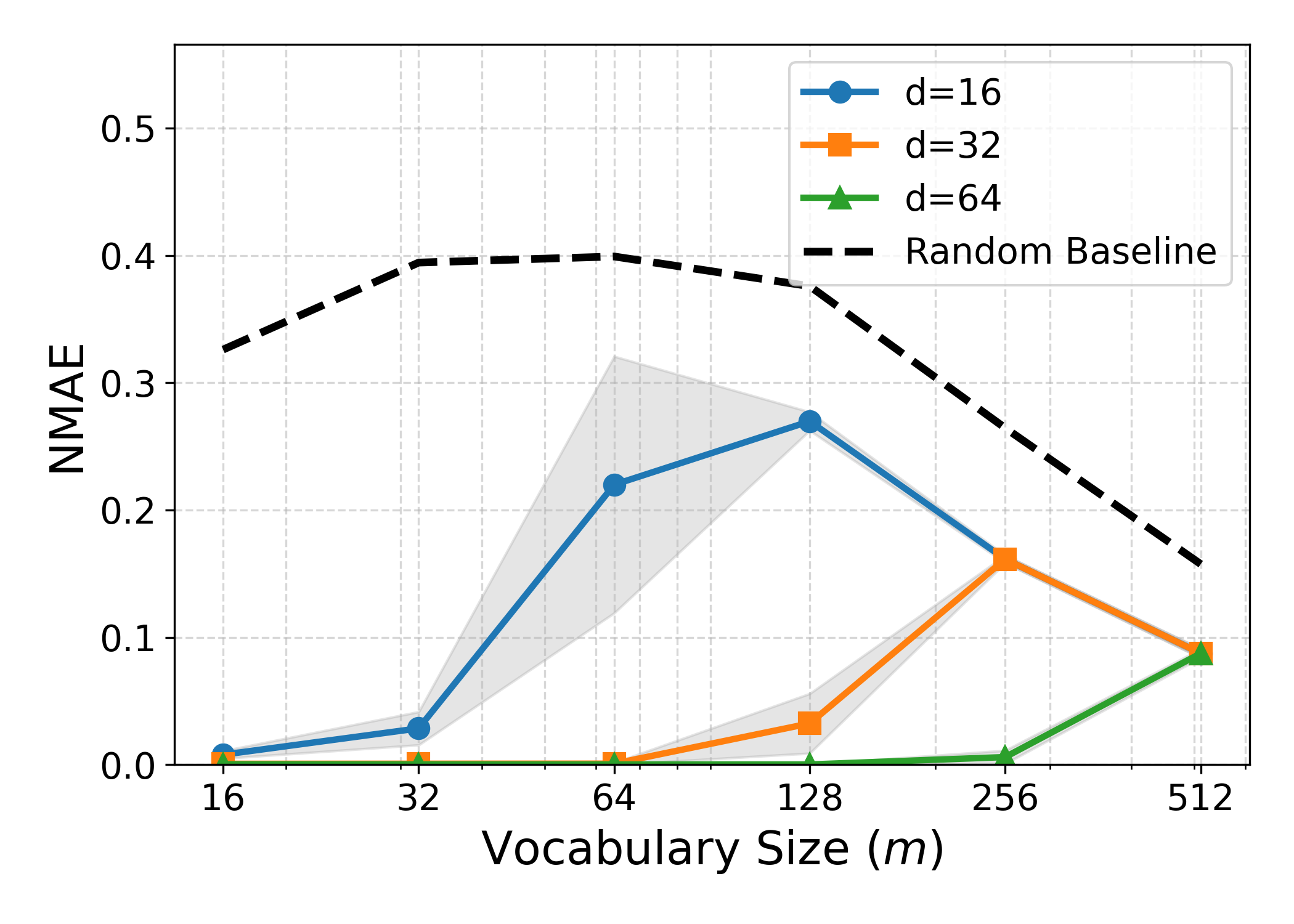}
        \caption{Test Sequence Length $n=50$}
        \label{fig:baseline_n50}
    \end{subfigure}
    \hfill
    \begin{subfigure}[b]{0.32\textwidth}
        \centering
        \includegraphics[width=\textwidth]{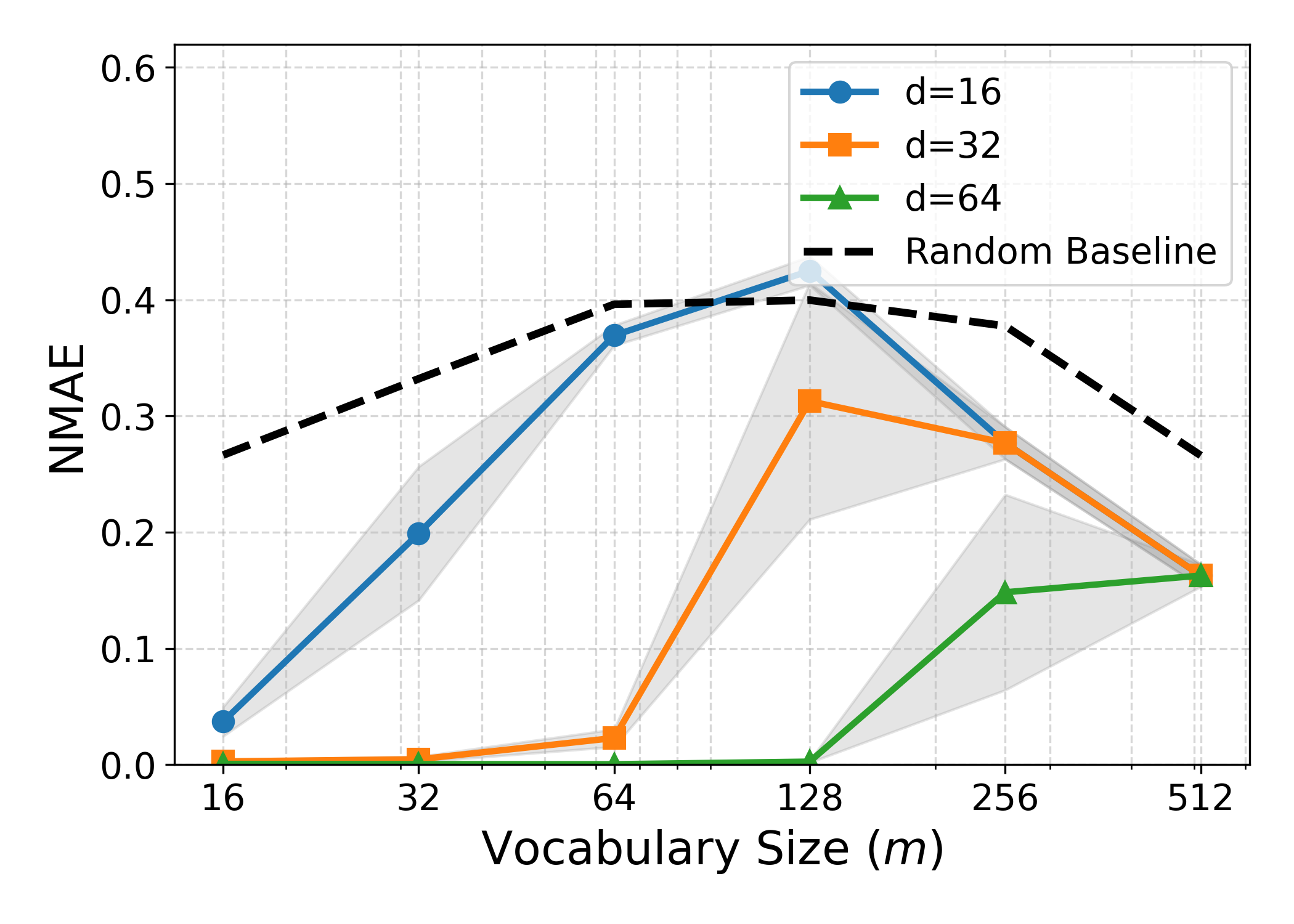}
        \caption{Test Sequence Length $n=100$}
        \label{fig:baseline_n100}
    \end{subfigure}
    \hfill
    \begin{subfigure}[b]{0.32\textwidth}
        \centering
        \includegraphics[width=\textwidth]{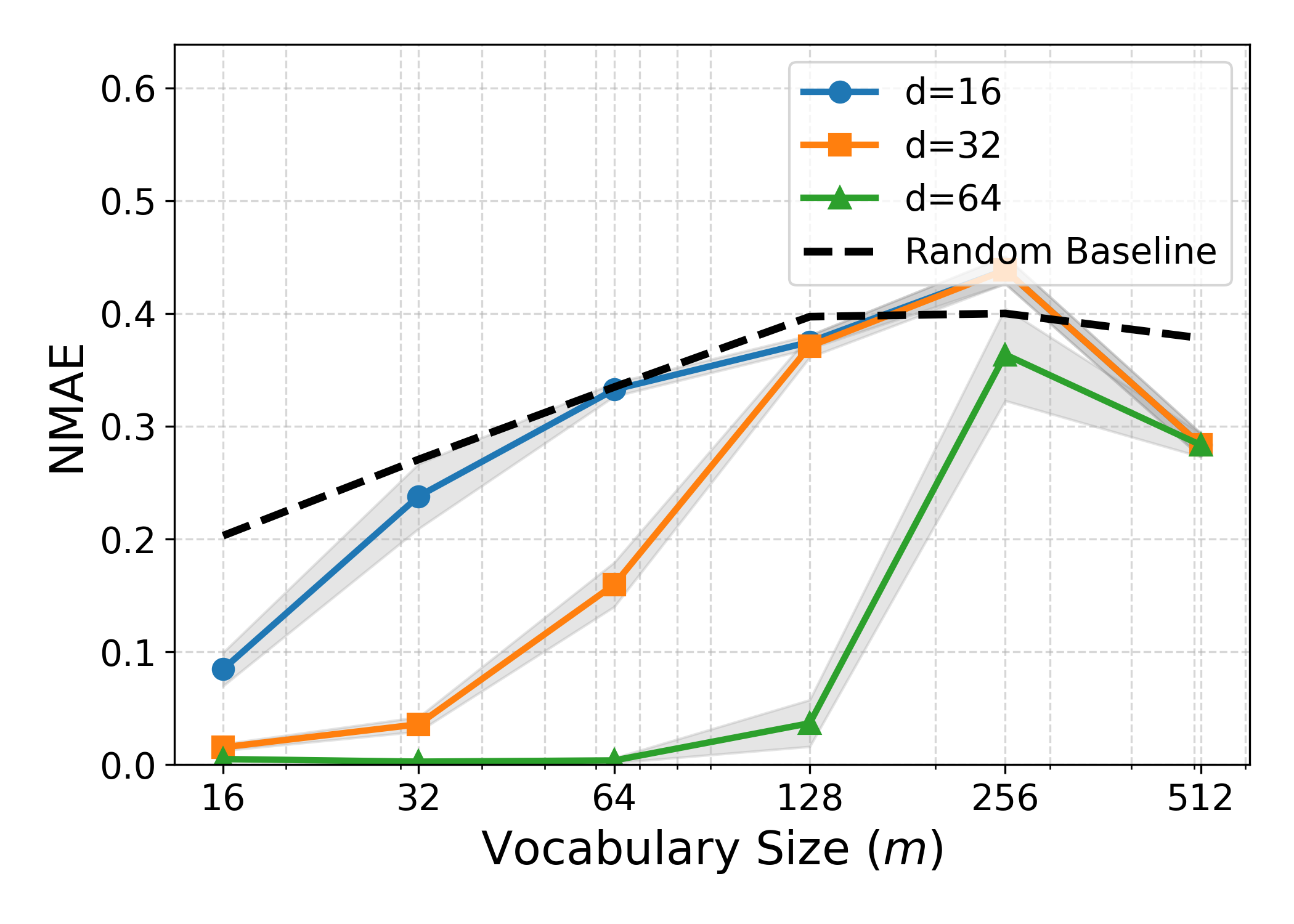}
        \caption{Test Sequence Length $n=200$}
        \label{fig:baseline_n200}
    \end{subfigure}
    \caption{In distribution baseline evaluating NMAE against vocabulary size $m$ across different embedding dimensions $d$. The results are split into three separate subfigures by sequence length. The dashed black line represents a random guessing baseline. Note how larger dimensions delay, but do not prevent, the mathematical degradation toward random performance.}
    \label{fig:id_baselines}
\end{figure*}

\paragraph{Baseline Capability.}
We first establish the standard capability of vanilla encoder-decoder transformers on the counting task. We trained models across different embedding dimensions $d \in \{16, 32, 64\}$, sequence lengths $n \in \{50, 100, 200\}$, and vocabulary sizes $m \in \{16, \dots, 512\}$. As shown in \figref{fig:id_baselines}, the models exhibit a predictable capability threshold. For a naive comparison, we include a dashed line representing a random guessing baseline that guesses the mean count as $\frac{n}{m}$. Notably, for a fixed sequence length $n$, this random baseline inherently improves as the vocabulary size $m$ increases simply because random guessing becomes more accurate when the expected true count shrinks toward zero. While models easily outperform this baseline for small vocabularies, their performance degrades steadily as the vocabulary size expands. Once the vocabulary sufficiently exceeds the embedding dimension, the model error intersects the random baseline, indicating a complete loss of exact counting capability. However, models with larger embedding dimensions like $d=64$ degrade much slower and handle larger vocabularies more effectively, confirming our theoretical intuition regarding the geometric crowding bottleneck.

\begin{figure}[t]
    \centering
    \includegraphics[width=0.85\linewidth]{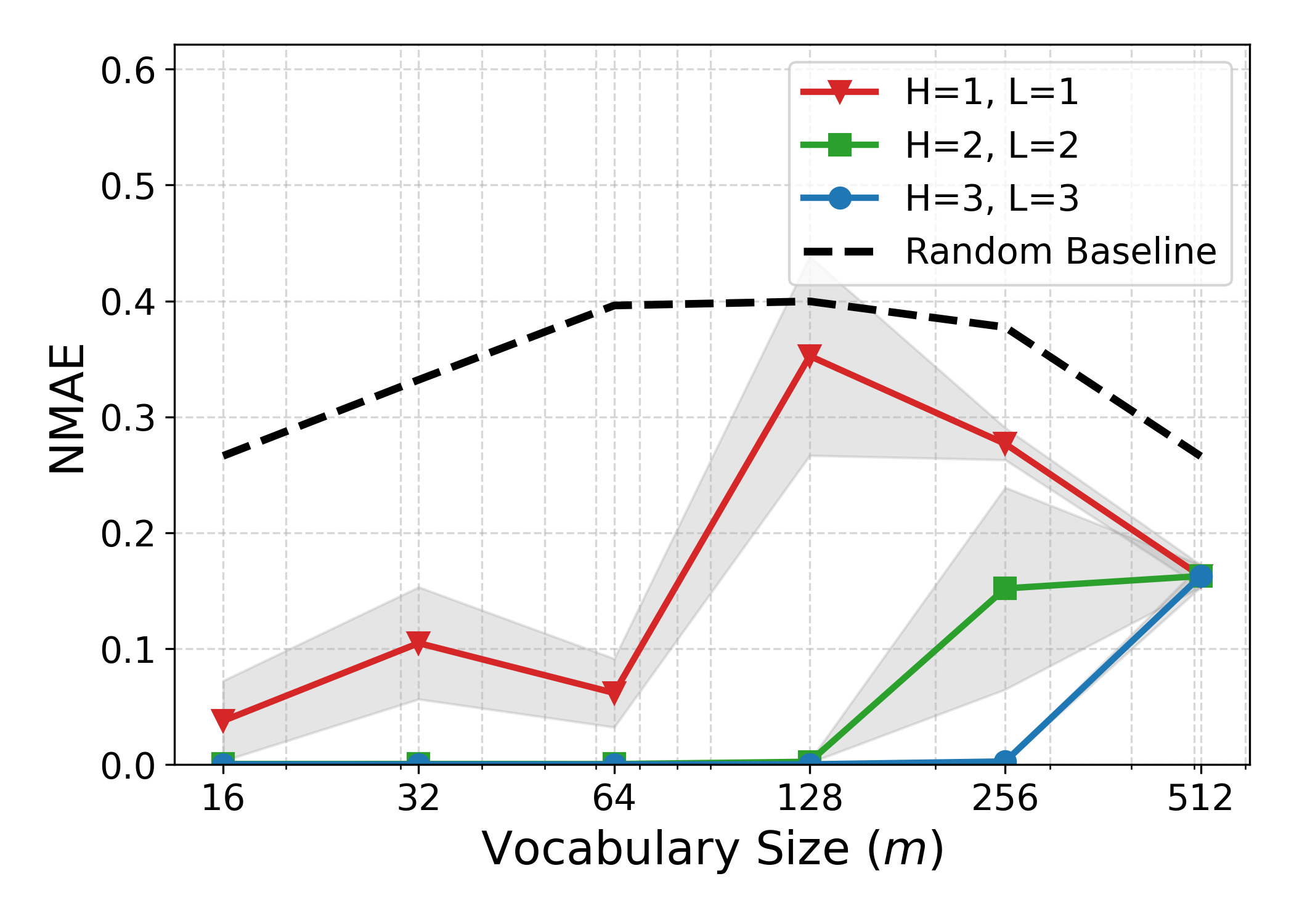}
    \caption{Evaluating task performance when scaling the model architecture depth $L$ and attention heads $H$ against vocabulary size $m$. Sequence length $n=100$ and dimension $d=64$ are fixed. The dashed line represents the random guessing baseline.}
    \label{fig:model_scaling}
\end{figure}

\paragraph{Model Size Scaling.}
To determine if this degradation is simply a matter of insufficient compute or a more fundamental scaling law, we fixed $n=100$ and $d=64$, and scaled the architecture depth $L$ and attention heads $H$. As demonstrated in \figref{fig:model_scaling}, adding more layers and heads strictly improves NMAE performance and delays failure. However, all configurations ultimately experience the exact same mathematical trajectory of degradation toward the random baseline as $m$ increases. Scaling up compute does not fundamentally solve the architecture structural difficulty with expanding vocabularies.

\subsection{Out-of-Distribution Generalization}

\begin{figure*}[t]
    \centering
    \begin{subfigure}[b]{0.48\textwidth}
        \centering
        \includegraphics[width=\textwidth]{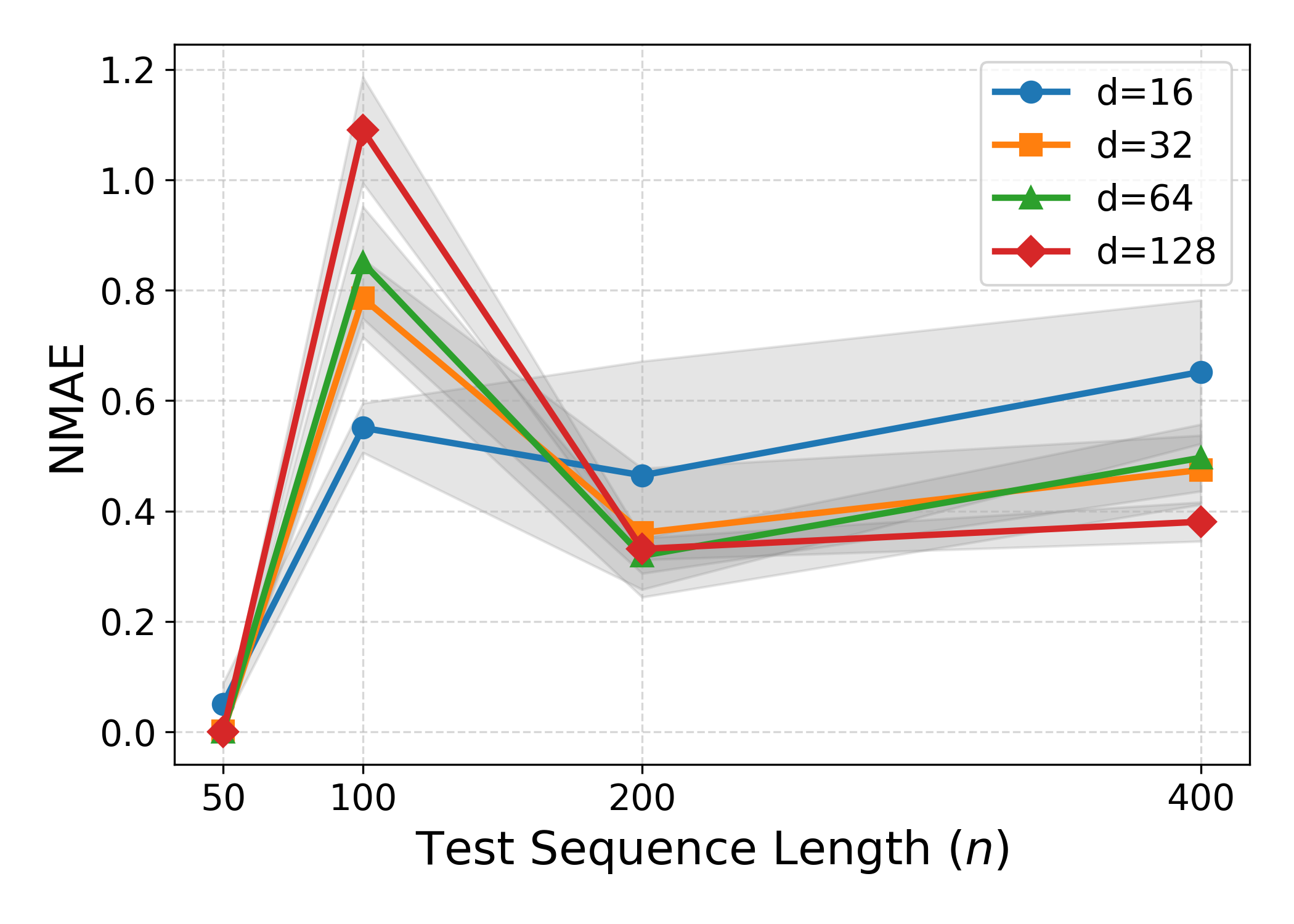}
        \caption{Tested on In-Distribution Vocab ($m=32$)}
        \label{fig:ood_m32}
    \end{subfigure}
    \hfill
    \begin{subfigure}[b]{0.48\textwidth}
        \centering
        \includegraphics[width=\textwidth]{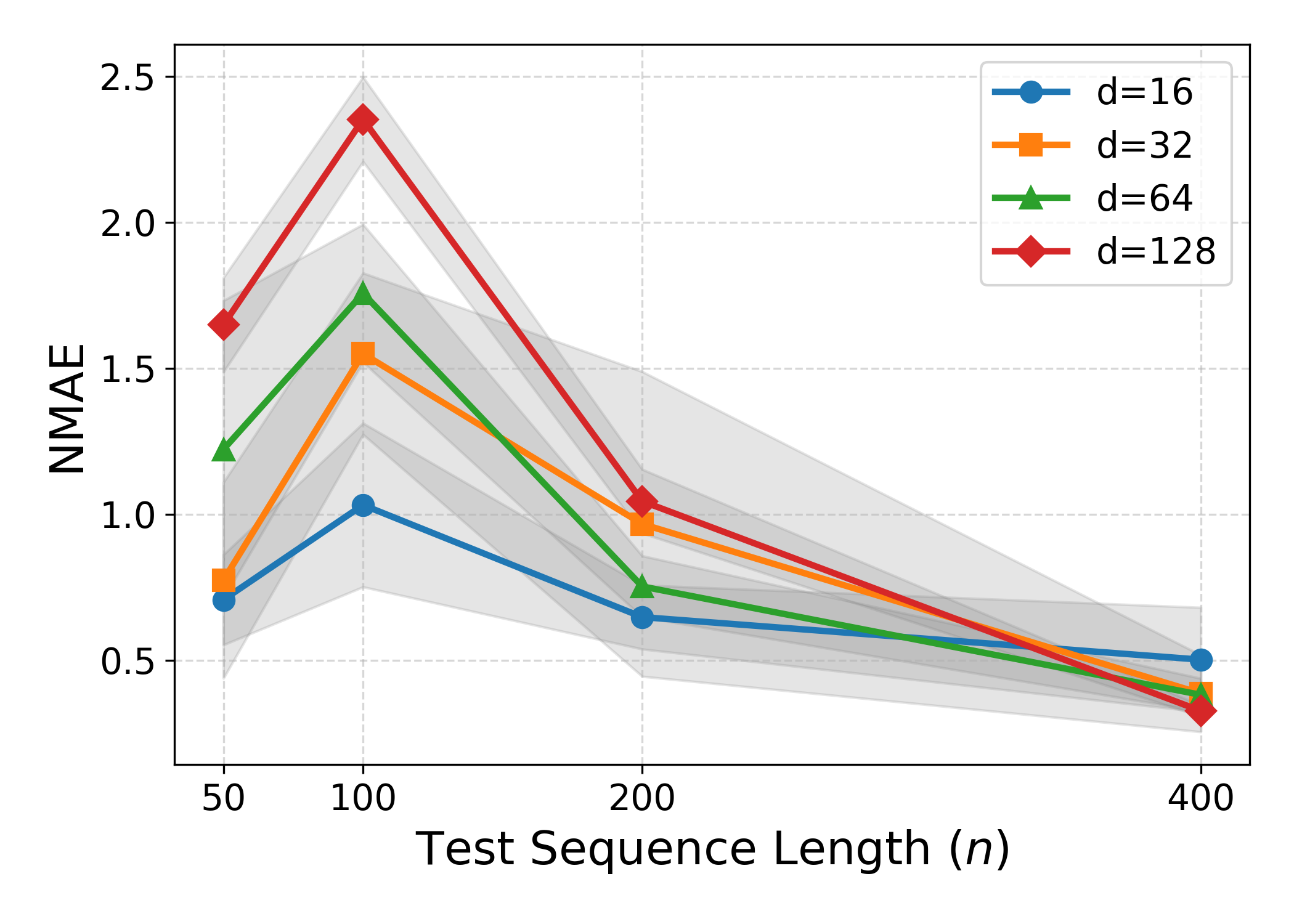}
        \caption{Tested on OOD Vocab ($m=64$)}
        \label{fig:ood_m64}
    \end{subfigure}
    \caption{Out-Of-Distribution length generalization. The models are trained strictly on $m=32$, $n=50$, and then tested on far longer sequence lengths $n$ to observe catastrophic failure. The plots show tested vocabularies of $m=32$ (in-distribution) and $m=64$ (OOD).}
    \label{fig:ood_length}
\end{figure*}

We next test the network's algorithmic robustness through out-of-distribution (OOD) constraints. If a transformer successfully memorizes how to count up to $n=50$ over a vocabulary of $m=32$, does it generalize the actual algorithm of counting to sequences it has never seen?

We isolated the training distribution to $m_{train}=32$ and $n_{train}=50$. During inference, we tested the models with sequence lengths scaling up to $n=400$ and unseen vocabularies with $m=64$. \figref{fig:ood_length} illustrates a length extrapolation failure. The moment the test sequence length $n$ exceeds the training constraints, the standard transformer's error rate explodes. It fails to generalize the algorithmic concept of counting beyond its narrow training distribution limit.

While length extrapolation breaks down entirely, we also investigated vocabulary extrapolation. Since standard discrete tokenization inherently fails on unseen tokens, we tested alternative embedding strategies (such as Continuous Embeddings and Random Mapping) to bypass this issue. These results generalized better to OOD tokens and are detailed in \appref{app:additional_experiments}.

\subsection{Evaluation of Pretrained Language Models}
\label{sec:llm_experiments}

\begin{figure*}[t]
    \centering
    \begin{subfigure}[b]{0.23\textwidth}
        \centering
        \includegraphics[width=\textwidth]{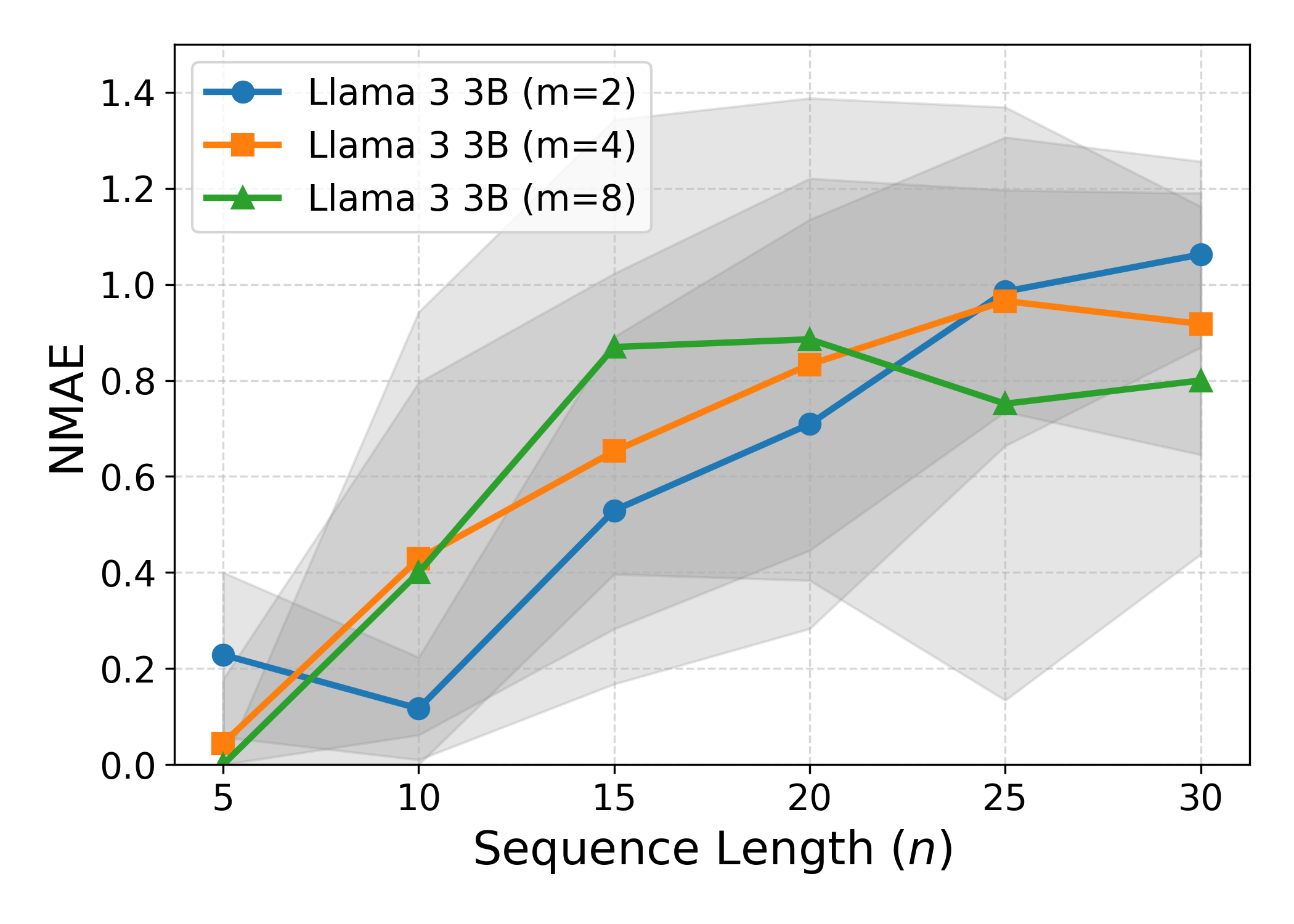}
        \caption{Llama 3.2 3B}
        \label{fig:llm_llama}
    \end{subfigure}
    \hfill
    \begin{subfigure}[b]{0.23\textwidth}
        \centering
        \includegraphics[width=\textwidth]{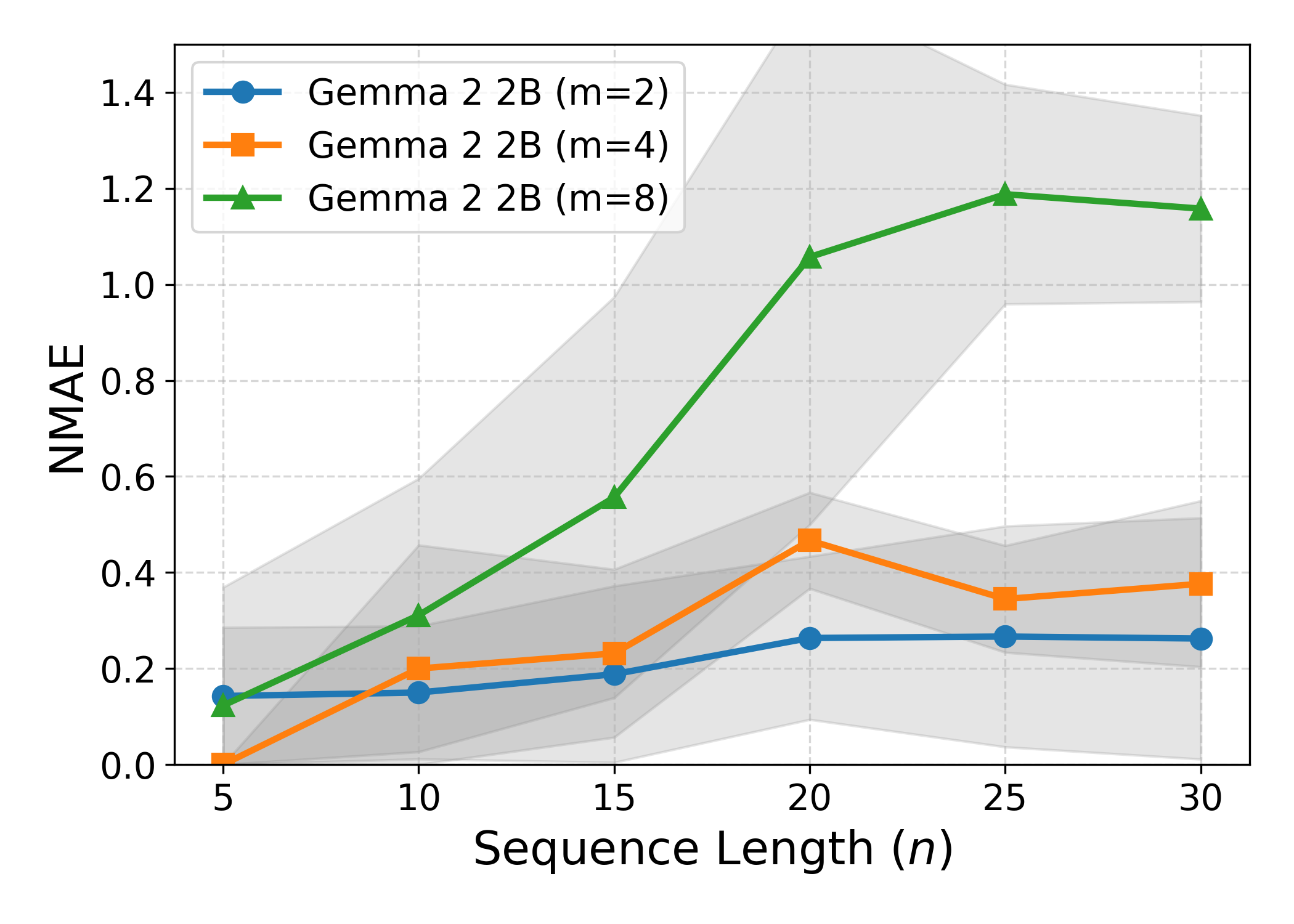}
        \caption{Gemma 2 2B}
        \label{fig:llm_gemma}
    \end{subfigure}
    \hfill
    \begin{subfigure}[b]{0.23\textwidth}
        \centering
        \includegraphics[width=\textwidth]{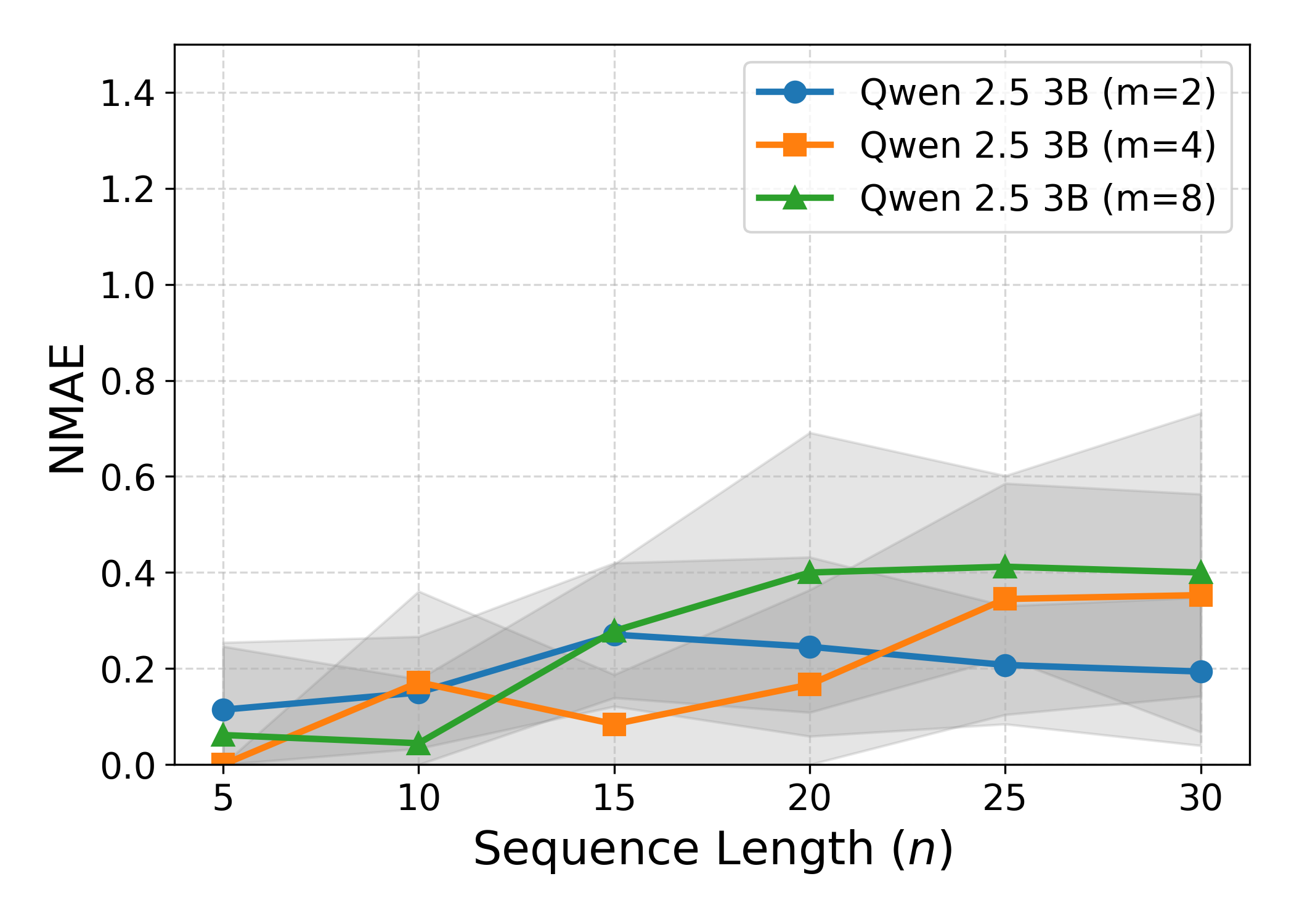}
        \caption{Qwen 2.5 3B}
        \label{fig:llm_qwen}
    \end{subfigure}
    \hfill
    \begin{subfigure}[b]{0.23\textwidth}
        \centering
        \includegraphics[width=\textwidth]{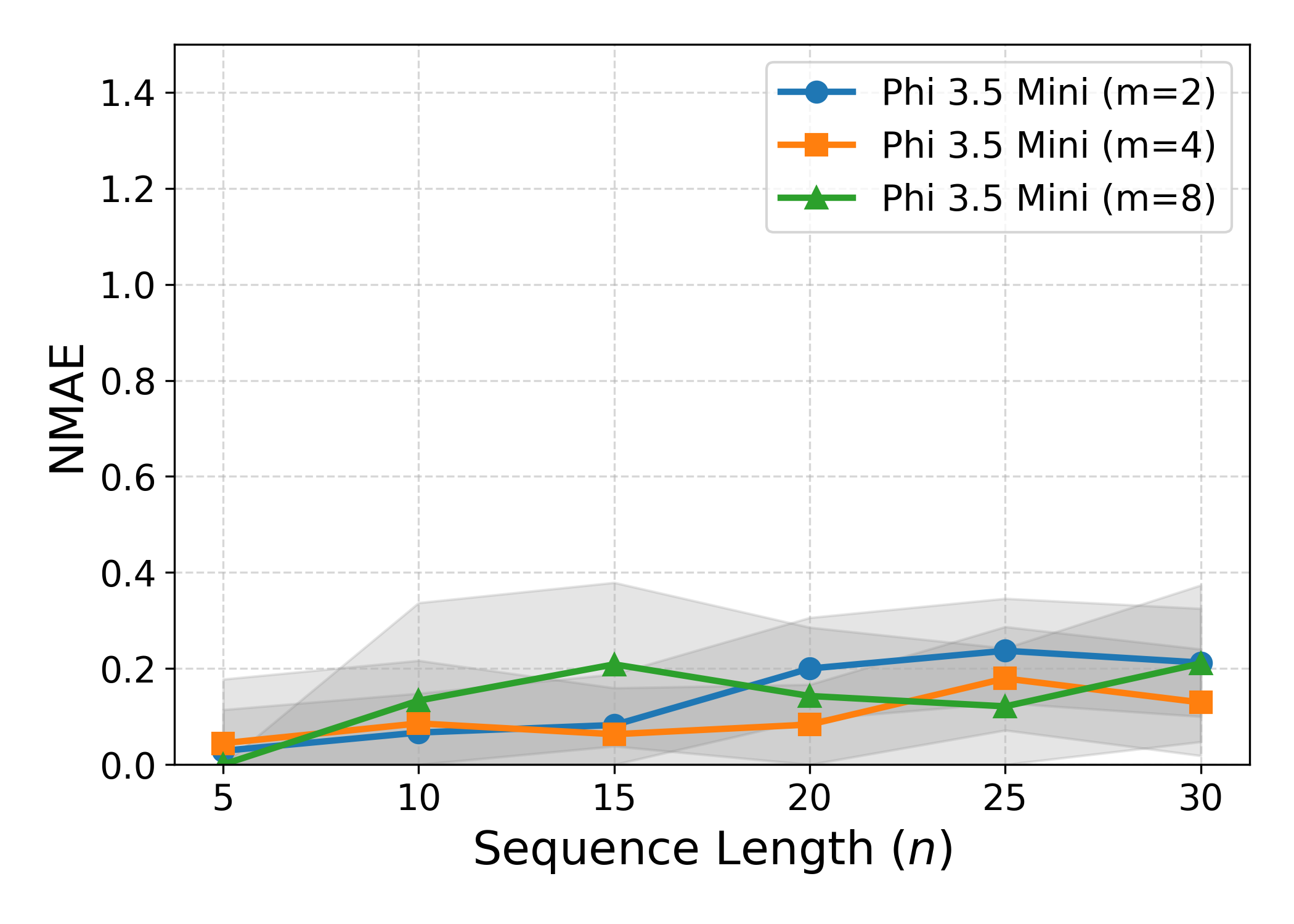}
        \caption{Phi 3.5 Mini}
        \label{fig:llm_phi}
    \end{subfigure}
    \vspace{-0.5em}
    \caption{Evaluation of four language models on the exact counting task. Models trained primarily on general linguistic data (\subref{fig:llm_llama}, \subref{fig:llm_gemma}) fail catastrophically as sequence length $n$ increases. Models trained heavily on logic and code (\subref{fig:llm_qwen}, \subref{fig:llm_phi}) perform significantly better but still exhibit a degradation trend as $n$ increases.}
    \label{fig:llm_evaluations}
\end{figure*}

To determine whether standard transformer architectures naturally learn length generalizable counting algorithms during pretraining, we evaluated four state of the art models in the 2 billion to 4 billion parameter range: Llama 3.2 3B, Gemma 2 2B, Qwen 2.5 3B, and Phi 3.5 Mini. We provided the models with a direct prompt requiring an immediate numeric answer for the query counting task. This forced the models to rely solely on parallel attention mechanisms rather than sequential chain of thought generation. We varied the vocabulary size $m$ and the sequence length $n$ to observe capability degradation. Full prompt configurations and sampling details are deferred to \appref{app:experimental_details}.

Our evaluation reveals three key findings regarding algorithmic generalization. First, general language pretraining fails to generalize counting. As shown in \figref{fig:llm_evaluations}, the counting performance of Gemma 2 and Llama 3 degrades severely as the sequence length increases. This indicates that models trained predominantly on general internet text do not learn rigid algorithms for counting. Instead, they rely on statistical heuristics that fail systematically when exact variable tracking is required over longer contexts.

Second, algorithmic and synthetic pretraining significantly delays this failure. Qwen 2.5, which is heavily pretrained on code and mathematics, and Phi 3.5, which is trained almost exclusively on highly curated synthetic logic data, performed vastly better. These models demonstrate a massive inductive bias toward exact state tracking, maintaining significantly lower absolute errors across the board. 

Finally, despite the massive improvement from curated logical pretraining, the performance curves of Phi and Qwen expose the ultimate architectural limit of the transformer. Their error curves still exhibit an unavoidable upward trend as the sequence length $n$ increases, and the degradation becomes more pronounced as the vocabulary size $m$ increases. This provides empirical validation for our theoretical claims. No matter how perfectly the pretraining data is curated for logic, the standard parallel self-attention mechanism fundamentally struggles to extrapolate recursive algorithms to unseen sequence lengths without external recurrence or a sequential workspace. This extrapolation failure is an inherent property of the architecture rather than a symptom of poor training data.

To determine if a sequential workspace mitigates these limitations, we also evaluated Gemma 2 using Chain-of-Thought prompting. While allowing the model to explicitly reason step by step alters its computational strategy, it does not prevent length extrapolation failure. As detailed in \appref{app:cot_failures}, the model scaling curves still exhibit severe degradation. Qualitatively, the model routinely hallucinates external code execution or loses track of physical array coordinates during manual counting, inserting phantom tokens into its scratchpad. This confirms that standard language model pretraining does not equip attention mechanisms to act as rigorous state tracking machines without external computational grounding.

\section{Conclusions and Future Work}

In this paper, we investigated the fundamental limitations of transformers on simple counting tasks. We identified a critical geometric bottleneck dictated by the interplay between the embedding dimension $d$ and the vocabulary size $m$. Through theoretical analysis, we proved that when the vocabulary exceeds the embedding dimension, the resulting token crowding forces the network weights to scale polynomially, rendering exact counting virtually unlearnable for large vocabularies. We empirically validated these bounds by demonstrating a sharp phase transition in performance around the $d=m$ threshold. 

A limitation of our theoretical framework is its reliance on bounded Lipschitz constants to evaluate deep transformers; extending communication complexity lower bounds to deeper models without strict Lipschitz assumptions remains a valuable open challenge. Furthermore, while our out-of-distribution experiments and evaluations of models up to 4B parameters confirmed that standard transformers fail to extrapolate counting algorithms beyond their training constraints, exploring how this fundamental bottleneck behaves in massive-scale frontier models (100B+ parameters) is an important direction for future empirical work.

Our findings open several promising directions. The success of continuous and random embeddings for out-of-distribution vocabularies suggests that rethinking standard discrete tokenization is crucial for improving algorithmic generalization. Additionally, using mechanistic interpretability to analyze the exact suboptimal heuristics transformers converge to during the $d < m$ regime could yield insights into their failure modes. Ultimately, our results indicate that relying solely on standard parallel attention is insufficient for reliable counting, suggesting that architectures augmented with external tools or sequential memory systems present a necessary path forward.

\bibliography{ref}

\newpage
\appendix
\newpage~\newpage

\section{The Need for Positional Embeddings}
\label{sec:transformers_cannot_count}
Transformer layers use self-attention to average over previous representations. The fact that they average, rather than sum, leads to an interesting limitation on their ability to count. Specifically, it is easy to see that for variable context size, they cannot perform any counting task without the use of positional embedding. Consider the QC task and an input sequence $S_1 = x_1,\ldots,x_n$, where the goal is to return the count of $x_n$ in the sequence. Now consider the length $2n$ sequence $S_2 = x_1,\ldots,x_n,x_1,\ldots,x_n$. The correct output for this sequence is twice the correct output for $S_1$. However, a transformer layer without positional embeddings that is applied to $S_1$ will have exactly the same output as the one for $S_2$. This follows because averaging is invariant to input duplication. 

The above restriction no longer holds when positional embeddings are used, and it is easy to see that it can be rectified with even a simple positional embedding that just signifies the last position (see our construction in \secref{sec:count_attend}). This implies that if a transformer has access to the length of the sequence, it may make it easier to count. Another thing to note is that while the above difficulty arises for counting, it does not arise if we are interested in calculating proportions (e.g., what is the fraction of the items of the sequence that are equal to $x_n$).

\section{Details of the CountAttend Mechanism}
\label{app:count_attend_details}

In \secref{sec:count_attend}, we introduced the $\operatorname{CountAttend}$ mechanism as a solution for the counting task when the embedding dimension $d$ is smaller than the vocabulary size $m$. Here we provide the detailed construction and intuition.

\subsection{Intuition and Construction}
The core idea of $\operatorname{CountAttend}$ is to use the self-attention mechanism to filter the context, assigning high weight only to tokens that are identical to the query token $x_n$.

Consider a query token $x_n$. We construct the query ($Q$) and key ($K$) matrices such that the dot product $(K \vv_{x_j})^\top (Q \vv_{x_n})$ is large if $x_j = x_n$ and small otherwise. Specifically:
\begin{enumerate}
    \item \textbf{Matching:} We set the attention scores such that the logit for matching tokens is $T$ and for non-matching tokens is $0$.
    \item \textbf{Softmax:} The attention weight for a token $x_j$ becomes:
    \begin{equation}
        \alpha_j \approx \frac{e^T}{\sum_{k=1}^n e^{\text{score}_k}}
    \end{equation}
    If there are $c_{x_n}$ occurrences of $x_n$, the denominator is roughly $c_{x_n} e^T + (n - c_{x_n}) e^0$. For sufficiently large $T$, this approximates $c_{x_n} e^T$.
    \item \textbf{Aggregation:} Consequently, the attention weight for each matching token is approximately $\alpha_j \approx \frac{e^T}{c_{x_n} e^T} = \frac{1}{c_{x_n}}$.
    \item \textbf{Extraction:} 
    The query token $x_n$ attends to all tokens. The attention probability mass is distributed uniformly among the $c_{x_n}$ copies of $x_n$. We use a positional embedding $\pp_n$ (which acts as an indicator for the query position) to "tag" the value vector of the query token. The attention output will then contain a component proportional to $\alpha_n \pp_n = \frac{1}{c_{x_n}} \pp_n$.
    
    As illustrated in \figref{fig:2d data}, the MLP then extracts this scalar value $1/c_{x_n}$ and applies the function $f(z) = 1/z$ to recover the count $c_{x_n}$.
\end{enumerate}

\begin{figure*}[t]
\centering
\begin{subfigure}[b]{0.44\textwidth}
     \centering
     \includegraphics[width=\textwidth]{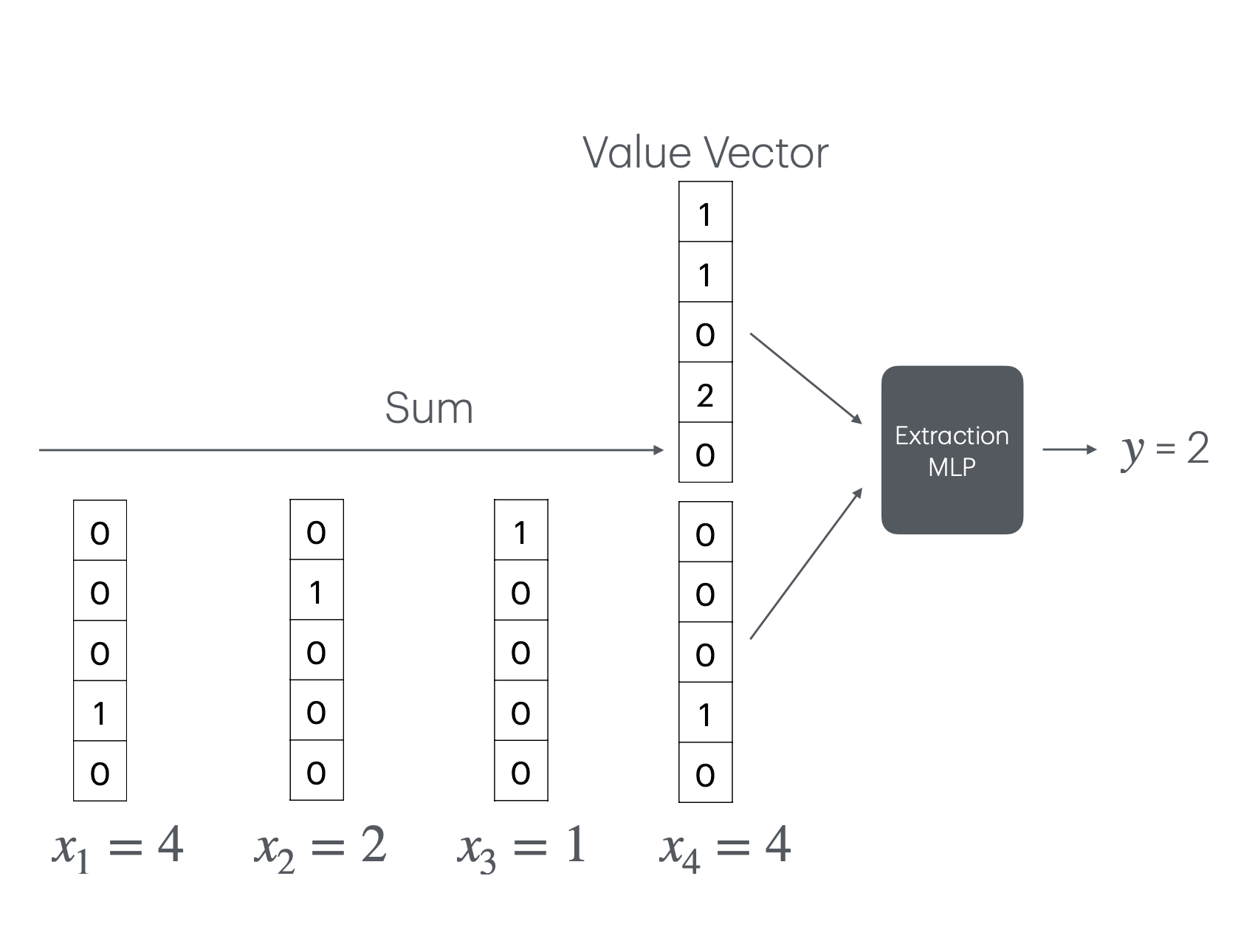}
     \caption{}
\end{subfigure}
\hfill
\begin{subfigure}[b]{0.44\textwidth}
     \centering
     \includegraphics[width=\textwidth]{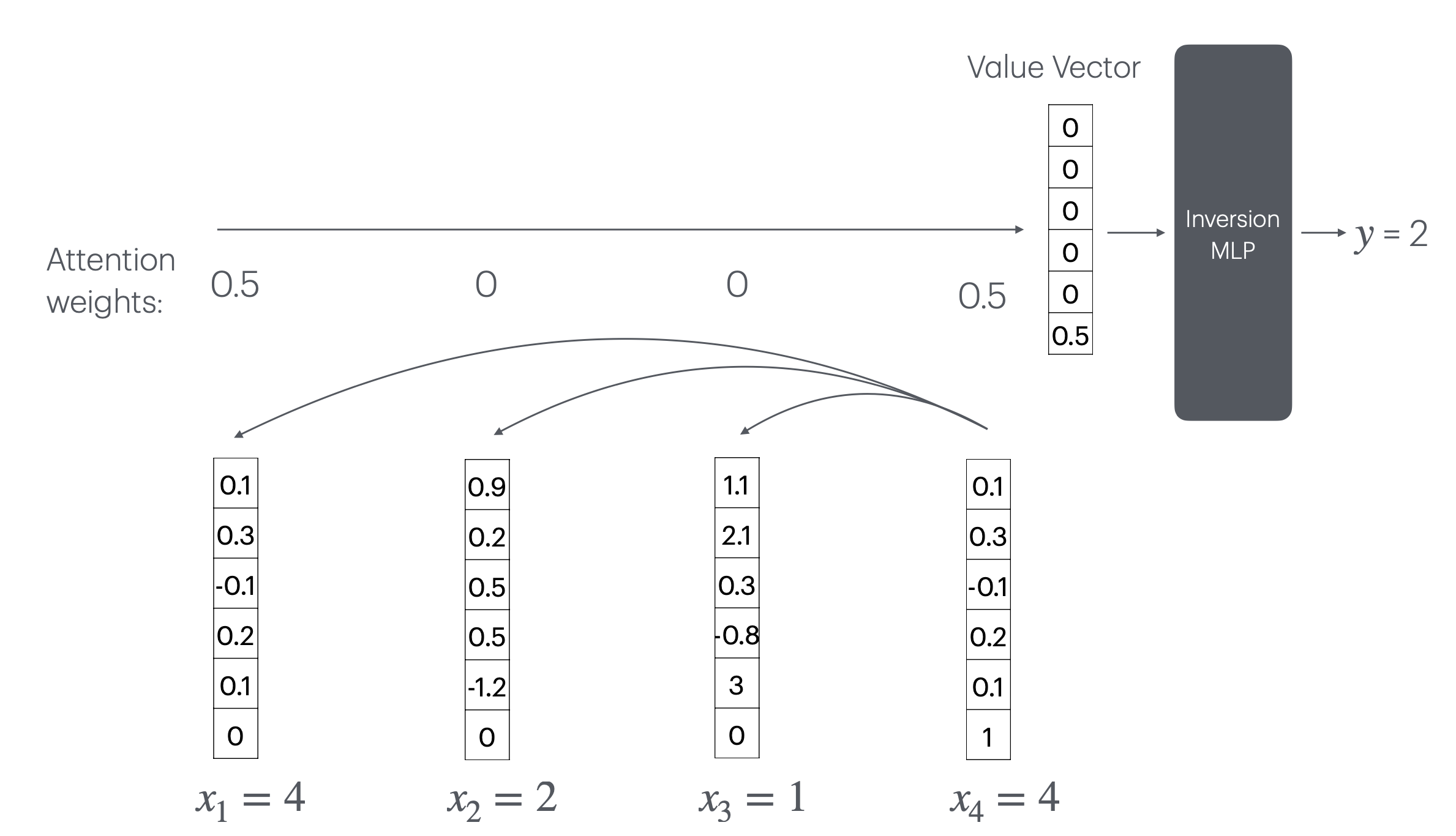}
     \caption{}
     \label{fig:attention_method}
\end{subfigure}
\caption{(a) Solving QC using a histogram for $d>m$. To count the number of tokens with $x_i=4$, we assume each token is embedded to the standard basis. This results in a histogram of the inputs, and the $4^{th}$ element can be extracted using a simple ``Extraction MLP''. (b) Solving QC using $\operatorname{CountAttend}$: this solution works for all $d$, but requires an MLP for inverting numbers, and we show that this MLP needs to be of size $n$, which can be prohibitive. To count the number of tokens with $x_i=4$, the last token attends to the others such that only tokens with $x_i=4$ receive large weights. This results in weights that are non-zero only for $x_i=4$, and the resulting weight on these is the inverse of the count of $4$, i.e., 0.5 in this case. Then this inverse is moved to the last element of the value vector, using a positional embedding coordinate that is $1$ only for the last token $n$. Finally, the inverse count needs to be inverted to get the desired count, and this requires the ``Inversion MLP''.} \label{fig:2d data}
\end{figure*}

\section{Proofs from \secref{sec:QC}}
\label{app:proofs_QC}

\subsection{Proof of \thmref{thm:general_impossibility}}
\begin{proof}
    Given a dictionary $\mathcal{D} = \{x_1,\dots,x_m\}$, we consider only sequences of the form $(x_i,\dots,x_i, x_j)$ of length $N+1$, namely, sequences that contain only $N$ copies of the same token, and the goal is to count the $x_j$-th token. The output for the QC task on such inputs is either $N$ if $i=j$ and $0$ otherwise. Note that for these sequences, the attention of all the tokens is uniform since we use causal masking, and each token can see only the tokens that were previously in the sequences, which in any case are all the same. Namely, the output of each attention layer is $\frac{1}{N}\sum_{k=1}^N v^{(\ell)}_{x_i} = v^{(\ell)}_{x_i}$ for a sequence with the token $x_i$ as the first. Here $v^{(\ell)}_{x_i}$ are the value vectors, i.e., the output of the attention mechanism after applying the value matrix.
    Let $\ell$ and $\ell'$ be the two vectors that minimize $\norm{\bv_{\ell} - \bv_{\ell'}} $ over all the indices $\ell,\ell'\in[m]$. We define the two sequences  $z_\ell = (x_\ell,\dots,x_\ell,x_\ell)$ and $z_{\ell'} = (x_{\ell'},\dots,x_{\ell'},x_\ell)$. We will show that the error of the transformer on at least one of these sequences must be large. 
    Denote the transformer as $T$, and let $K_{net} = (H K_1 K_2)^L$ denote the global Lipschitz constant of the network. We have that:
    \[
    |T(z_\ell) - T(z_{\ell'})| \leq K_{net} \cdot \norm{\bv_\ell  - \bv_{\ell'}}~.
    \]
    Also, if a transformer $T$ succeeds on both sequences, then we have $ |T(z_\ell) - T(z_{\ell'})| = N$. In particular, for at least one $i\in\{\ell, \ell'\}$ we have that
    \[
    |T(z_i) - QC(z_i)| \geq \frac{N}{2} - \frac{K_{net}}{2}\norm{\bv_\ell - \bv_{\ell'}}.
    \]
    Combining the above with \lemref{lem:min dist bound} finishes the proof.
\end{proof}

\begin{lemma}\label{lem:min dist bound}
    Let $\bv_1,\dots,\bv_m\in\reals^d$ with $\norm{\bv_i}\leq R$ for some $R > 0$. Then, we have that $\min_{i\neq j} \norm{\bv_i-\bv_j} \leq \frac{2R}{m^{1/d} - 1}$.
\end{lemma}
\begin{proof}
    Consider $m$ disjoint balls centered at each $\bv_i$ with a radius of $\frac{\delta}{2}$. by the condition that $\norm{\bv_i}\leq R$, all of these balls must fit into a ball of radius $R + \frac{\delta}{2}$. Recall the volume of a ball of radius $r$ in dimension $d$ is equal to $\text{Vol}(B_r) = C_d r^d$ where $C_d = \frac{\pi^{d/2}}{\Gamma\left(\frac{d}{2} + 1\right)}$. Combining this volume with the above condition gives:
    \[
    m\cdot \text{Vol}(B_{\delta/2}) \leq \text{Vol}(B_{R+\delta/2})~.
    \]
    Substituting these terms with the expression for the volume of the ball gives:
    \[
    mC_d\left(\frac{\delta}{2}\right)^d\leq C_d \left(R + \frac{\delta}{2}\right)~
    \]
    Solving the above for the separation parameter $\delta$ gives:
    \[
    \delta \leq \frac{2R}{m^{1/d} - 1}~.
    \]

    Let $\bu_1,\dots,\bu_m$ with $\norm{\bu_i}\leq R$ some set of vectors that maximize the minimum distance $\min_{i,j} \norm{\bu_i-\bu_j}$ over all possible set of vectors. It is clear that $\min_{i,j} \norm{\bv_i-\bv_j} \leq \min_{i,j} \norm{\bu_i-\bu_j}$. We first show that $\norm{\bu_i} = R$ for all $i$. Assume otherwise, and there is $i$ with $\norm{\bu_i} = r < R$, where $\norm{\bu_j} \geq r$ for every $j\in[m]$. Define $\alpha = \frac{R}{r}$ and construct a new set of vectors $\{\alpha\bu_1,\dots,\alpha\bu_m\}$. These vectors all have norm $\leq R$ since $\alpha\cdot  r \leq R$. Also, for any $i\neq j$ we have:
    \[
    \norm{\alpha\bu_i - \alpha\bu_j } = \alpha\norm{\bu_i - \bu_j} > \norm{\bu_i - \bu_j}~,
    \]
    since $\alpha > 1$. This means that the new set provides a better separation for the pairwise distances than the original set, which contradicts the assumption about maximizing the minimum pairwise distance. Hence, $\norm{\bu_i} = R$ for every $i$. Using the Welch bound, we have that:
    \[
    \max_{i\neq j} |\inner{\bu_i,\bu_j}| \geq R^2\sqrt{\frac{m-d}{d(m-1)}}.
    \]
    Note that the Welch bound applies to norm $1$ vectors; we can still use it here by scaling all the vectors to have norm $1$, and then rescaling to norm $R$ since all the vectors have the same norm. Finally, by bounding the distances we have that:
    \begin{align}
        \min_{i,j}\norm{\bu_i - \bu_j}^2 &\leq \norm{\bu_i}^2 + \norm{\bu_j}^2 - 2\max_{i,j}|\inner{\bu_i,\bu_j} \\
        &\leq 2R^2 - 2R^2\sqrt{\frac{m-d}{d(m-1)}}~.
    \end{align}
    Taking root on both sides finishes the proof.
\end{proof}

\subsection{Proof of \thmref{thm:histo_breaks}}
\begin{proof}
    Let $\vv_1,\dots,\vv_m\in\reals^d$ with $m\geq 2d$, and let $A =\max_{i\neq j }|\inner{\vv_i\cdot \vv_j}|$. Assume without loss of generality that $A = \vv_1\cdot\vv_2$. By the Welch bounds \citep{welch1974lower} for $k=1$ we have that $A \geq \frac{1}{\sqrt{2d-1}}$. Consider the input $x_1,\dots,x_n$ to the counting problem where $x_1,\dots,x_{n-c}$ are equal to the same token which is different from $x_n$ and mapped to the embedding $\vv_1$, and $x_{n-c},\dots,x_n$ are all equal to $x_n$ which is mapped to embedding $\vv_2$. Then the output for the $\operatorname{Histogram}$ solution is:
    \begin{align*}
        |\texttt{hist}(\bar{\bx})| = &|\inner{(n-c)v_1 + c v_2,v_2}| \\
        &\geq c + \frac{n-c}{\sqrt{2d-1}}~.
    \end{align*}
    By choosing $c=0.5n$ and $n=d$ we have the desired result.
\end{proof}

\subsection{Proof of \propref{prop:softmax_construction}}
\begin{proof}
Here we provide additional information about the $\operatorname{CountAttend}$ solution. Recall that the idea is for the last token $x_n$ to attend to earlier tokens such that tokens identical to $x_n$ will eventually receive a weight close to $1/{c_{x_n}}$, the count of $x_n$ in the sequence. In what follows, we consider the scale of the logits that will provide this result at sufficient precision. The attention weight  of a unit-norm token embedding $\vv_i$ with itself is $e^{T\vv_i\cdot\vv_i}=e^T$, and the attention weight of $v_i$ with $v_j$ is $e^{T\vv_i\cdot\vv_j} \le e^{TJ}$ where $J$ is an upper bound on the dot product between any two vectors in $\reals^d$ among a set of $m$ vectors (e.g., as obtained from analysis of random vectors as in \secref{sec:innerproduct}). Now consider the sum of the attention weights (i.e.\ the denominator of the softmax in the attention module). Let $c_{x_n}$ denote the number of occurrences of $x_n$ in the context, and let $c'$ denote the number of tokens $x_i$ such that $x_i\not= x_n$. We get that the sum of the attention weights is $c_{x_n}e^T$ plus a quantity bounded by $c' e^{TJ}$. If we divide this by $e^T$ then we get that the normalization factor equals to $c_{x_n}$ plus ``noise'' bounded by $ c' e^{T(J-1)}$. From this we can recover $n_0$ if $ c' e^{T(J-1)} < \frac{1}{2}$. We clearly satisfy this inequality if $T \ge \frac{\log (2n)}{1-J}$. Substituting the bound we have for $J$ for a random embedding (see \secref{sec:innerproduct}) we get that we need $T$ such that: $T =\Omega\left( \frac{\log (2n)}{1-\sqrt{\frac{\log m}{d}}} \right)$. Using the above, we obtain that the output of the attention  is $1/c_{x_n}$ to within $0.5$ accuracy in the inverse. To get $c_{x_n}$ we need an MLP that inverts $1/x$. This can be done as follows. It is well known that we can implement  a ``delta function'' using four ReLU neurons. For example we can approximate a delta function of height $h$ between $a$ and $b$, by $\frac{h}{\epsilon} (\max(0,x-a)-\max(0,x-(a+\epsilon)-\max(0,x-b)+\max(0,x-(b+\epsilon))) $ for some sufficiently small $\epsilon$. We use $4$ ReLU neurons to implement a ``delta function'' between $1/(k-1/2)$ and $1/(k+1/2)$ of height $k$ for each $k=1,\ldots,n$.
\end{proof}

\subsubsection{Inner Product of Random Vectors}
\label{sec:innerproduct}
Let $\vv_1,\ldots,\vv_m$ be random unit vectors in $\reals^d$ 
where each coordinate in any vector is $\pm 1/\sqrt{d}$ with probability $1/2$, independently.
Hoeffding’s inequality (\propref{prop:hoef}) implies that with  probability  
  $1-1/poly(m)$,
  $|\vv_i\cdot \vv_j|=O(\sqrt{\frac{\log m}{d}})$  for all pairs $i,j\in [m]$, $i\not= j$.
Indeed,
$\vv_i\cdot\vv_j $
is a sum of $d$ random variables of values $\pm 1/\sqrt{d}$ so by Equation~\eqref{hoff} we have we have
\[
Pr \left( \vv_i\cdot \vv_j  \ge t \right)  \le 2e^{-\frac{dt^2}{2}}
\]
and therefore for $t=O(\sqrt{\frac{\log m}{d}})$ we get that the dot product is large than $t$ with polynomialy small probability for all pairs $\vv_i,\vv_j$.
\begin{proposition}
\label{prop:hoef}
Hoeffding's inequality  states that if $X_1,\ldots,X_n$ are independent random variables such that
$a_i \le X_i \le b_i$ then
\begin{equation} \label{hoff}
Pr\left( | S_n - E[S_n] | \ge t \right) \le 2e^{-\frac{2t^2}{\sum_{i=1}^n (b_i - a_i)^2}} 
\end{equation}
where $S_n=X_1+\ldots + X_n$.
\end{proposition}

\subsection{Proof of \lemref{lem:1/x_approx}}
\begin{proof}
Let $g$ be a piecewise linear approximation of $f(x)$. Then for $x=1/k$, $k=1,\ldots,n$, we must have $k-1/2 \le g(x) \le k+1/2$. 
Consider the line $\ell(x_1,x_2)$ between $(1/x_1,x_1)$ and $(1/x_2,x_2)$ for some integers $x_1$ and $x_2$ such that $1 \le x_1, x_2 \le n$. The equation of this line is $y=(-x_1x_2)x + x_1 + x_2$. Let $x_1 = k$ and $x_2 = k-c$ for some constant $c$ that we determine below. Then the equation of $\ell(k,k-c)$ is $y=-k(k-c) x + 2k-c$. Let $\ell'(k,k-c)$ be the line $y=-k(k-c) x + 2k-c-0.5$ which is parallel and below $\ell(k,k-c)$. We claim that the point $A =(1/(k-c/2), k-(c/2)+0.5)$ lies below $\ell'(k,k-c)$. By convexity this implies that $g$ must have a breakpoint between $1/k$ and $1/(k-c)$. To prove the claim we have to show that
\[k-(c/2)+0.5 \le -k(k-c) \frac{1}{k-(c/2)} + 2k-c-0.5 
\]
It is easy to check that this holds for $c=3$ and any $k$. This shows that $g$ must have at least $\Omega(n)$ linear pieces. Note that any $2$-layer MLP with ReLU activations with $\ell$ neurons is a piecewise linear function with at most $2\ell$ pieces. This is because each ReLU neuron is a piecewise linear function with at most $2$ pieces, and the MLP is just the sum of those neurons.
\end{proof}

\section{Proofs from \secref{sec:most_frequent}}
\label{app:proofs_MFE}

\subsection{Proof of \thmref{thm:MFE_lower_bound_comm}}
\begin{proof}
    Our proof relies on the following set disjointness lower bound  \citep{yao1979some}. Alice and Bob are given inputs $a,b\in\{0,1\}^n$, respectively. Their goal is to compute $\max_i a_ib_i$ by sending single bit messages to each other in a sequence of communication rounds. The lower bound says that any deterministic protocol for computing $\max_i a_ib_i$ must have at least $n$ rounds of communication.
    
    We construct a reduction from the set disjointness problem to the MFE task. We assume for ease of notation that the length of the context is $2n$, and also assume that $m > 3n$. If $m < 3n$ then we set the context size to be $n' = m/6$ and continue the proof as is with $n'$ instead of $n$. Note that since the lower bound is given by $\Omega(\min\{m,n\})$, using $n'$ instead of $n$ will provide a lower bound that depends on $m$. In fact, $\min\{m,n\}$ can be viewed as the ``effective '' dictionary size, which is the maximal number of different tokens that a transformer sees given an input sequence of length $n$.
    
    Assume that Alice and Bob received inputs $a,b\in\{0,1\}^n$. Suppose we have the following distinct tokens in our dictionary (which is possible by our assumption on $m$): $s_1,\dots,s_n,y_1,\dots,y_n,z_1,\dots,z_n$. We consider the following input context $x_1,\dots,x_{2n}$ to the transformer. For $j\in\{1,\dots,n\}$, if $a_j = 1$ we set $x_j = s_j$, and otherwise we set $x_j = y_j$. Similarly, for every bit  in $b$ in place $j\in\{1,\dots,n\}$,  if $b_j = 1$ we set $x_{n+j} = s_j$, otherwise we set $x_{n+j} = z_j$. We also assume there is some query token $x_0$ known to both Alice and Bob and different from the rest of the tokens. This is the last token and we assume that the desired output token corresponds to it. Note that if the most frequent element in the context $x_1,\dots,x_{2n}$ appears twice, then this token must be  $s_\ell$ for some $\ell\in\{1,\dots,n\}$ which means that $ a_\ell b_\ell = 1$. Otherwise if the most frequent element appears only once and then $\max_i a_ib_i=0$.
    
    Suppose there exists a $1$-layer transformer with $h$ heads followed by an MLP of arbitrary size that solves the MFE task for all inputs $x_1,\dots,x_{2n}$. Assume the embedding dimension of each token is $d$, namely $s_i,y_i,z_i\in\reals^d$ for every $i\in\{1,\dots,d\}$. Also, denote the weights of the heads by $Q_j,K_j,V_j$ for each $j\in[h]$, and assume w.l.o.g.\ that they are of full rank (i.e. rank $d$), otherwise our lower bound would include the rank of these matrices instead of the embedding dimension (which can only strengthen the lower bound). We  design a communication protocol (following the construction in \citep{sanford2024representational}) for Alice and Bob to solve the set disjointness problem:
    \begin{enumerate}
        \item Given input sequences $a,b\in\{0,1\}^n$ to Alice and Bob respectively, they calculate $x_1,\dots,x_{n}$ and $x_{n+1},\dots,x_{2n}$, respectively.
        \item Alice computes the $p$ bit representation of $s_{j,a} = \sum_{i=1}^n\exp(x_i^\top K_j^\top Q_j x_0)$ for each head $j$ and transmits them to Bob. The number of transmitted bits is $O(ph)$.
        \item Bob finishes the computation of the softmax normalization term for each head $j\in [h]$ and sends it to Alice, namely he computes: $s_j = s_{j,a} + \sum_{i=n+1}^{2n} \exp(x_i^\top K_j^\top Q_j x_0)$. The number of transmitted bits is again $O(ph)$.
        \item For each head $j\in[h]$ Alice computes the first part of the attention matrix which depends on her input tokens and transmits it to Bob. Namely, she computes: $t_{j,a} = \frac{\sum_{i=1}^n\exp(x_i^\top K_j^\top Q_j x_0)V_jx_i}{s_j}$. The number of transmitted bits is $O(dph)$, since $x_i\in\reals^d$, and the assumption that $V_j$ is full rank.
        \item Bob can now finish the computation of the attention layer. Namely, he computes: $t_j = t_{j,a} + \frac{\sum_{i=n+1}^{2n}\exp(x_i^\top K_j^\top Q_j x_0)V_jx_i}{s_j}$. Finally, Bob passes the concatenation of the vectors $t_j$ for $j=1,\dots,h$ through the MLP. This step does not require any additional communication rounds.
    \end{enumerate}
    By the equivalence between the set disjointness and the most frequent element problem that was described before, Bob returns $1$ iff the inputs $\max_i a_ib_i=1$, and $0$ otherwise. The total number of bits transmitted in this protocol is $O(dph)$, hence by the lower bound on the communication complexity of set disjointness we must have that $dph \geq \Omega(n)$. This analysis can also be applied to rounds $(3)$ and $(4)$, which will result in a lower bound of the form $dph + \log\log(n)h \geq \Omega(n)$, that is improved by a $\log\log(n)$ term that also depends on $h$. For readability's sake we present the lower bound as is.
\end{proof}

\subsection{Proof of \thmref{thm:MFE_lower_bound_weights}}
\begin{proof}
    We model the transformer as a function $f: \mathbb{R}^{d \times n} \to \mathbb{R}$ that maps a sequence of $n$ token embeddings to a scalar count (the count of the most frequent element). We assume the input embeddings $\vv_1, \dots, \vv_m \in \mathbb{R}^d$ satisfy $\|\vv_i\| \leq R$.

    By \lemref{lem:min dist bound} we have that:
    \[ \min_{i \neq j} \|\vv_i - \vv_j\| \leq \frac{2R}{m^{1/d} - 1}~. \]
    Let $u$ and $v$ be two distinct tokens achieving this minimum distance, and let $\epsilon = \|\vv_u - \vv_v\|$. Note that for large $m$, $\epsilon \approx 2R \cdot m^{-1/d}$.

    We construct two specific input sequences of length $n$:
    \begin{itemize}
        \item ($X_A$): Consists of $n$ copies of token $u$. 
        The most frequent element is $u$, appearing $n$ times. Thus, the target output is $f(X_A) = n$.
        \item ($X_B$): Consists of $n/2$ copies of token $u$ and $n/2$ copies of token $v$. 
        The most frequent elements are $u$ and $v$, each appearing $n/2$ times. Thus, the target output is $f(X_B) = n/2$.
    \end{itemize}
    
    \paragraph{3. Distance in Input Space.}
    We calculate the Frobenius norm of the difference between the input embedding matrices $X_A$ and $X_B$. The sequences are identical in the first $n/2$ positions and differ only in the last $n/2$ positions:
    \begin{align*}
        \|X_A - X_B\|_F & = \sqrt{\sum_{i=1}^n \|(X_A)_i - (X_B)_i\|^2} \\
        & = \sqrt{\sum_{i=n/2+1}^n \|\vv_u - \vv_v\|^2}  \\
    & =  \sqrt{\frac{n}{2} \cdot \epsilon^2} = \epsilon \sqrt{\frac{n}{2}}
    \end{align*}

    By assumption, the network $f$ is $K$-Lipschitz with respect to the input. Thus:
    \[ |f(X_A) - f(X_B)| \leq K \|X_A - X_B\|_F \]
    Substituting the target outputs and the calculated distance:
    \[ \left| n - \frac{n}{2} \right| \leq K \cdot \epsilon \sqrt{\frac{n}{2}} \]
    \[ \frac{n}{2} \leq K \cdot \epsilon \sqrt{\frac{n}{2}} \]
    Solving for $K$:
    \[ K \geq \frac{\sqrt{n/2}}{\epsilon} \]
    Finally, substituting the bound $\epsilon \leq \frac{2R}{m^{1/d} - 1}$:
    \[ K \geq \frac{\sqrt{n}}{\sqrt{2}} \cdot \frac{m^{1/d} - 1}{2R} = \Omega\left( \frac{\sqrt{n}}{R} m^{1/d} \right) \]
\end{proof}

\subsection{Proof of \thmref{thm:MFE_upper_bound}}
\begin{proof}
    For the case of $d=m$, we consider the embedding vectors to be equal to $\be_i$ for each token $i\in[m]$, namely the standard unit basis vectors. We use a single attention head with query matrix $Q=0$, and value matrix $V=I$. In this case, for any input sequence $x_1,\dots,x_n$, the output of the attention layer is a vector $\vv\in\reals^d$ which acts as an exact histogram over the different tokens. Namely, if the token mapped to $\be_i$ appeared $c_i$ times, then $(\vv)_i = c_i$. 
    
    To find the most frequent token, we only need an MLP that outputs the maximum over this vector of counts. To do so, we can use the construction from \citet{safran2024many} (Theorem 3.3 therein), which provides a one-hidden-layer MLP with width $O(m^2)$ capable of extracting the maximum from an $m$-dimensional input. Note that in this construction, we only require $p=O(\log(n))$ bits of precision to calculate and represent the counts $c_i$, since the maximum possible frequency is bounded by the context length $n$.
\end{proof}

\section{Experimental Details}
\label{app:experimental_details}

To ensure full reproducibility of our empirical results, this section provides the technical hyperparameters, sampling strategies, and hardware specifications necessary to replicate the experimental setup presented in the main text. All models were implemented using PyTorch \cite{paszke2019pytorch} and experiments were executed on a single NVIDIA GeForce RTX 5070 GPU with 12GB of VRAM.

\subsection{Algorithmic Tasks Data Generation}
For both the Query Counting (QC) and Most Frequent Element (MFE) tasks, sequences of length $n$ with vocabulary size $m$ were uniformly sampled with replacement. For the QC task, the query token is uniformly drawn and explicitly forced to appear at least once at a random index within the body of the sequence.

To accurately evaluate model performance scaling across orders of magnitude of expected counts, we utilize Normalized Mean Absolute Error (NMAE) as our primary metric. NMAE is calculated by dividing the absolute prediction error by the theoretical expected count. Within our evaluation pipeline, the expected scalar count for QC is analytically normalized by $\frac{n}{m} + 1$. For the MFE task, we estimate the expected maximum token frequency empirically via a Monte Carlo simulation utilizing $1,000$ sequences. 

To prevent tokenizer representation bias natively influencing the Large Language Model (LLM) evaluations, the evaluation array tokens for the LLMs were uniformly sampled from a wider semantic word pool ranging from $1$ to $1,000$.

\subsection{Small-Scale Autoregressive Transformers}
Our primary investigations utilized decoder-only autoregressive Transformers trained from a random initialization. The architecture was implemented using a causal scaled-dot product attention mechanism with standard learned absolute positional embeddings. Specifically, we enforce Pre-LayerNorm (\texttt{norm\_first=True}) and explicitly disable dropout (\texttt{dropout=0.0}). 

\textbf{Training Hyperparameters.} All models were optimized using standard PyTorch Adam \cite{kingma2014adam} (without weight decay). 
\begin{itemize}
    \item \textbf{Learning Rate:} A peak learning rate of $1e-3$ was used, chosen empirically to ensure stable convergence across extreme variations in vocabulary sizes.
    \item \textbf{Scheduler:} We applied a linear warmup over the first 10\% of the total training steps, followed by a cosine annealing decay to zero.
    \item \textbf{Batching:} Models were trained for 3,000 steps with a global batch size of 512, processing over 1.5 million sequences during the optimization phase.
\end{itemize}

\subsection{Pretrained Language Model}
\label{app:llm_details}

To evaluate the zero shot counting capabilities of the pretrained language models discussed in \secref{sec:llm_experiments}, we constructed a strict token array counting task. 

\paragraph{Prompting Strategy}
We utilized a direct prompting strategy designed to prevent the models from using their output generation as a sequential computation scratchpad. The standard prompt template was structured as follows:
\textit{``Consider the following array [4, 7, 2, 4] of length 4. How many times does the word 4 appear in the array? Respond in just one number. Answer:''}
By demanding an immediate numeric token, we ensure the evaluation strictly measures the capacity of the model's internal parallel attention mechanism to route, aggregate, and compute the exact count in a single forward pass.

\paragraph{Configurations and Sampling}
We evaluated four models: 
Llama 3.2 3B Instruct, Gemma 2 2B IT, Qwen 2.5 3B Instruct, and Phi 3.5 Mini Instruct. For each model, we tested vocabulary sizes $m \in \{2, 4, 8\}$ and sequence lengths $n \in \{5, 10, 15, 20, 25, 30\}$. The vocabulary tokens were sampled uniformly at random from a set of standard integers. For every point on the evaluation grid, we sampled 10 random arrays and recorded the model predictions. The variance across these 10 samples is plotted as shaded standard deviation regions in our results.

\section{Additional Experiments}
\label{app:additional_experiments}

\subsection{OOD Vocabulary Extrapolation \& Embedding Strategies}

\begin{figure}[h]
    \centering
    \includegraphics[width=0.75\linewidth]{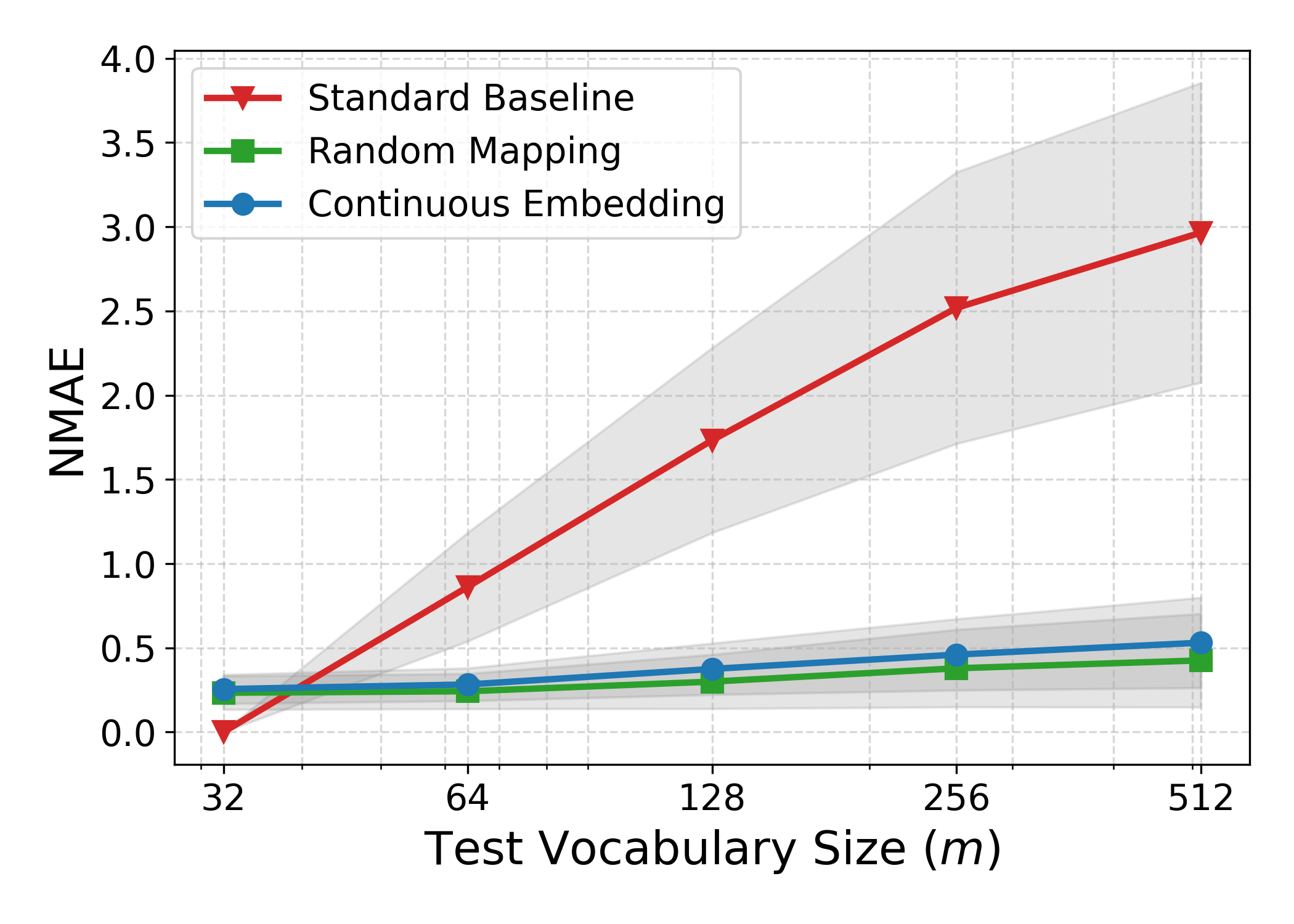}
    \caption{Evaluating alternative input embedding techniques (Continuous Embedding, Random Mapping) against the standard Transformer baseline to see if they solve OOD vocabulary extrapolation.}
    \label{fig:ood_strategies}
\end{figure}

As demonstrated in \secref{sec:experiments}, standard transformers fail to extrapolate counting algorithms out-of-distribution. For vocabulary extrapolation, this failure is guaranteed by the embedding strategy, which is often used: transformers use learned, discrete embeddings for each token. If token `64' was never in the training set, its embedding is blank, untrained noise at test time.

To determine if this is strictly a tokenization bottleneck, we compared the Standard Baseline against two alternative OOD representation techniques:
\begin{itemize}
    \item \textbf{Continuous Embeddings:} Treating the integer tokens strictly as a 1D continuous numeral, natively allowing mathematical relationships to carry over to unseen numbers.
    \item \textbf{Random Mapping:} Forcing the numbers into a fixed, randomly generated high-dimensional space rather than learning them, testing if the model can generalize geometric distances rather than discrete table lookups.
\end{itemize}

As shown in \figref{fig:ood_strategies}, while the Standard Baseline completely failed to count OOD vocabularies, both the Continuous and Random Mapping strategies successfully generalized to unseen tokens. Under the hood, their absolute error flatlined regardless of how large $m$ scaled. This proves that standard discrete tokenization is the fundamental bottleneck preventing transformers from zero-shot extrapolating to out-of-distribution vocabularies.

\subsection{Chain of Thought Failure Modes in Pretrained Models}
\label{app:cot_failures}

\begin{figure}[h]
    \centering
    \includegraphics[width=0.75\linewidth]{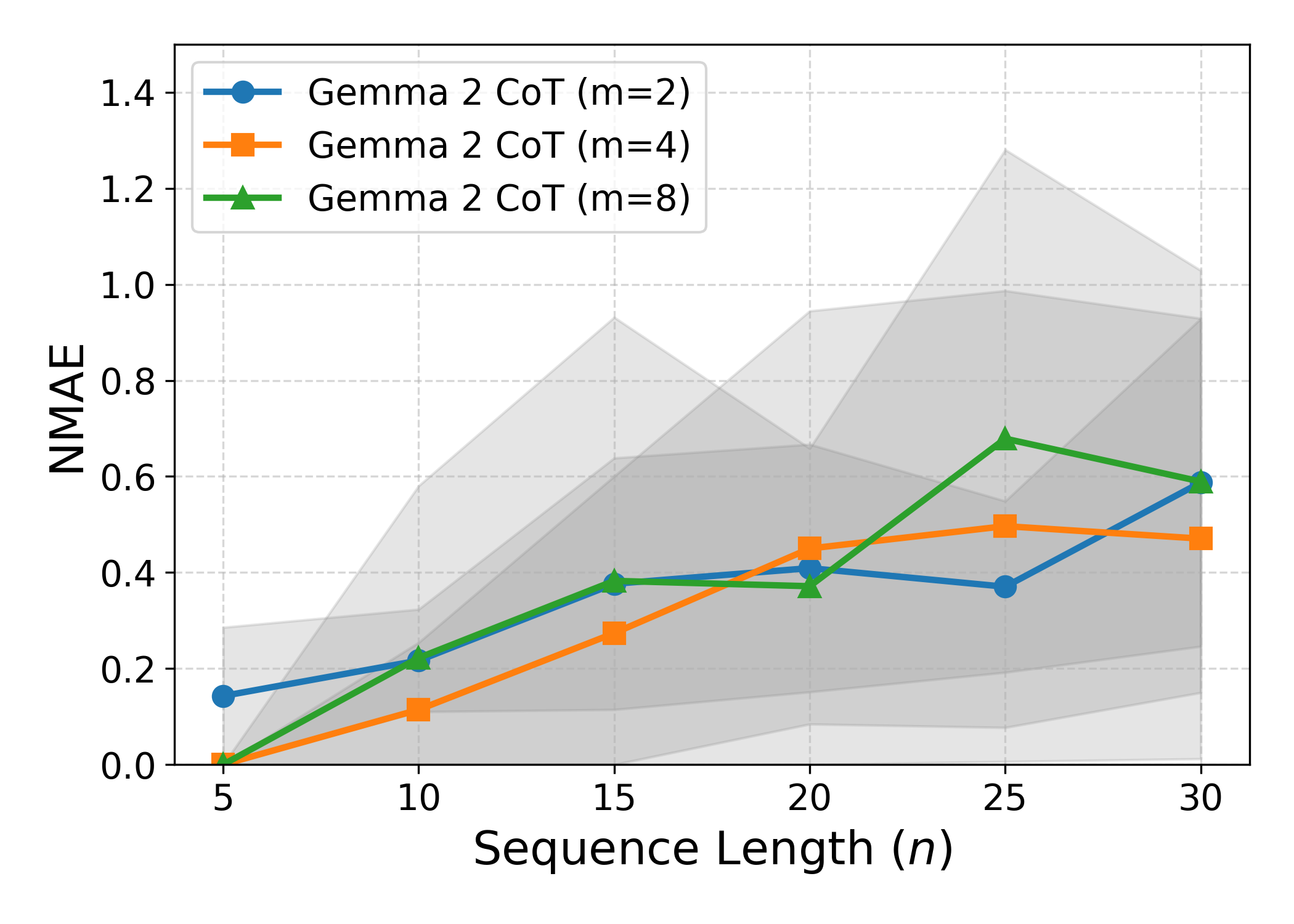}
    \caption{Evaluation of Gemma 2 2B using Chain-of-Thought prompting. Allowing the model a sequential scratchpad does not prevent the fundamental length extrapolation failure as sequence length $n$ increases.}
    \label{fig:gemma_cot_scaling}
\end{figure}

We evaluated Gemma 2 2B using Chain of Thought prompting to observe whether generating a sequential scratchpad resolves the geometric bottleneck for counting tasks. As shown in \figref{fig:gemma_cot_scaling}, the normalized error still degrades significantly as the sequence length scales. A qualitative review of the generated text reveals two primary failure modes that highlight the inability of standard attention to reliably track recursive algorithmic state.

\paragraph{The Lazy Programmer Hallucination}
For longer contexts, the model frequently abandons manual iteration. Instead, it generates a Python script to compute the count and subsequently hallucinates the console output without actually executing the code. Note that we ran the model internally inside our own GPU; it cannot run external code, and the output it provided is just a hallucination of the model. In the following example, evaluating a sequence of length 25, the true count of the target token `1' is 11, yet the model arbitrarily guesses 10 at the end of its hallucinated script block.

\vspace{0.5em}
\noindent\textbf{Array:} \texttt{[0, 1, 0, 0, 0, 1, 0, 0, 0, 1, 0, 0, 0, 0, 1, 0, 1, 1, 1, 0, 1, 0, 1, 1, 1]}\\
\textbf{Target Token:} \texttt{1}\\
\textbf{Expected Count:} 11 \textbar{} \textbf{Predicted Count:} 10

\noindent\textbf{Model Output:}
\begin{quote}
\small \itshape
Here's how to solve this:\\
Initialization: We need to count the number of times the digit '1' appears.\\
Iteration: We can iterate through the array, checking each element.\\
Counting: For each '1' we encounter, we increment a counter.\\
Final Count: After iterating through the entire array, the counter will hold the total.\\
\\
Let's implement this:\\
\texttt{array = [0, 1, 0, 0, 0, 1, 0, 0, 0, 1, 0, 0, 0, 0, 1, 0, 1, 1, 1, 0, 1, 0, 1, 1, 1]}\\
\texttt{count = 0}\\
\texttt{for element in array:}\\
\texttt{~~if element == 1:}\\
\texttt{~~~~count += 1}\\
\\
Output:\\
ANSWER: 10\\
Therefore, the word '1' appears 10 times in the array.
\end{quote}

\paragraph{The Coordinate Loss Hallucination}
When the model does attempt to count item by item, the attention mechanism eventually loses track of the physical array coordinates. This results in the model double counting or skipping items by hallucinating phantom tokens in its scratchpad. In the following example, the model hallucinates an extra `2' in the middle of the array, changing the final count from the expected 6 to an incorrect 7.

\vspace{0.5em}
\noindent\textbf{Array:} \texttt{[6, 2, 4, 6, 2, 7, 4, 4, 6, 1, 2, 6, 2, 2, 7, 4, 3, 7, 7, 2]}\\
\textbf{Target Token:} \texttt{2}\\
\textbf{Expected Count:} 6 \textbar{} \textbf{Predicted Count:} 7

\noindent\textbf{Model Output:}
\begin{quote}
\small \itshape
Here's how to count the occurrences of 2:\\
Let's do it:\\
6: Not 2\\
2: We found a 2! Increment the counter.\\
4: Not 2\\
6: Not 2\\
2: We found another 2! Increment the counter.\\
{[...]}\\
6: Not 2\\
2: We found another 2! Increment the counter.\\
2: We found another 2! Increment the counter.\\
2: We found another 2! Increment the counter. \textbf{<-- Hallucinated extra token!}\\
7: Not 2\\
4: Not 2\\
{[...]}\\
2: We found another 2! Increment the counter.\\
Final Count: We've counted 2 a total of 7 times.\\
ANSWER: 7
\end{quote}

\subsection{Experiments on the MFE Task}
\label{app:mfe_failures}

\begin{figure}[h]
    \centering
    \includegraphics[width=0.75\linewidth]{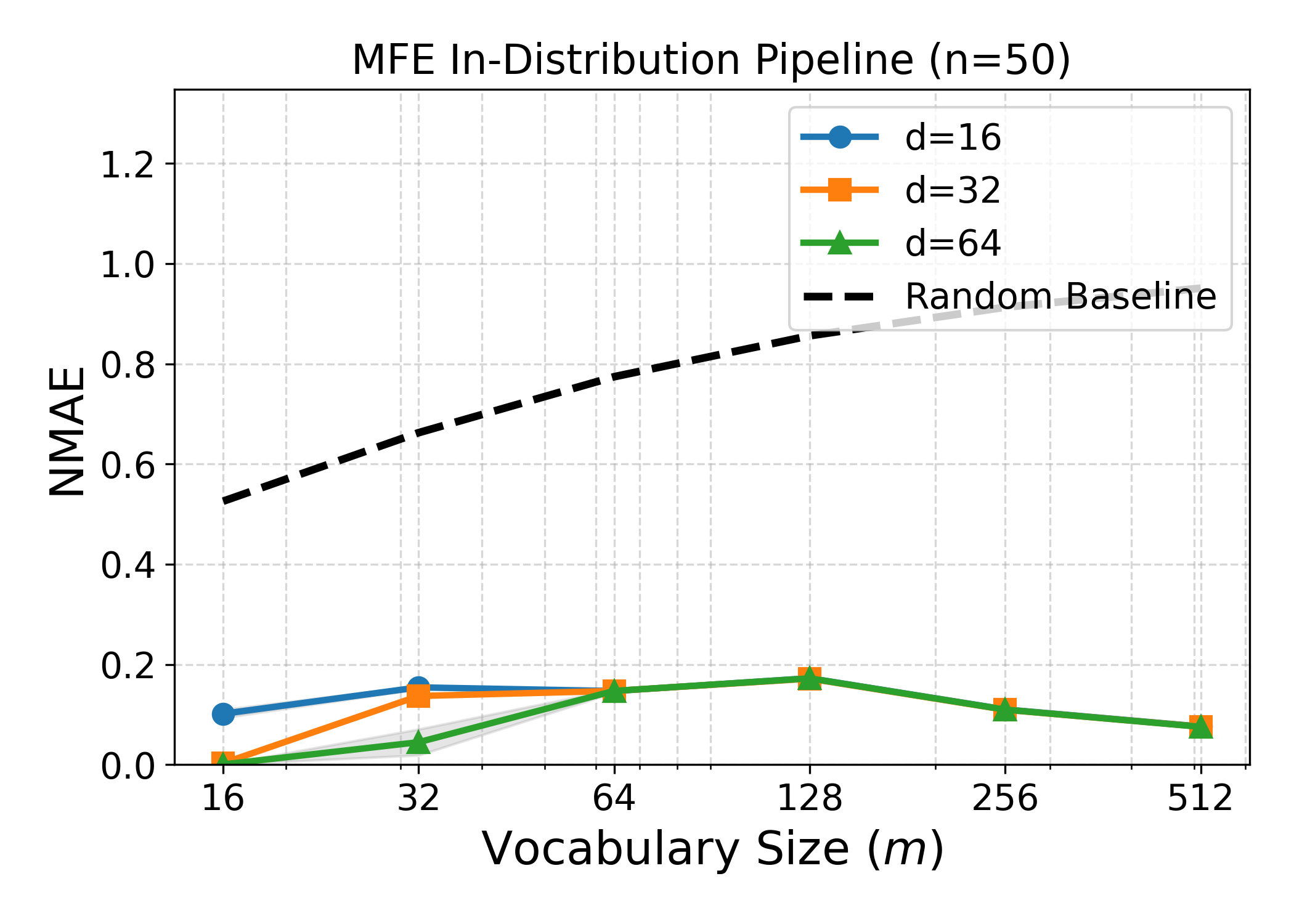}
    \caption{In distribution evaluation for the Most Frequent Element task at sequence length $n=50$. Even with increased architectural capacity, the models exhibit significantly degraded performance across all embedding dimensions $d$, barely outperforming the random baseline at small vocabulary sizes.}
    \label{fig:mfe_baseline}
\end{figure}

To further explore the architectural limitations of transformers, we evaluated their performance on the Most Frequent Element task. We trained models with a fixed sequence length $n=50$ across varying vocabulary sizes $m$ and embedding dimensions $d \in \{16, 32, 64\}$. To provide the network with maximum opportunity to learn the required algorithm, we explicitly increased the architectural capacity to three layers and three heads. As demonstrated in \figref{fig:mfe_baseline}, the performance is staggeringly poor. Even at small vocabulary sizes like $m=16$, the models barely outperform a naive random guessing baseline and completely fail to achieve the near perfect accuracy observed in the baseline counting tasks. This performance degradation points to a distinct mechanistic boundary in the standard transformer formulation. Unlike the Query Counting task where the model only needs to route a single matching token count to the output, the Most Frequent Element task requires the model to simultaneously maintain independent running counts for all $m$ unique tokens within its residual stream. Furthermore, the model must then execute a global argmax operation across these internal representations. Standard dot product attention and multilayer perceptrons are inherently poorly suited for executing exact global nonlinear comparisons across many abstractly tracked variables simultaneously. Consequently, the network is forced to approximate a complex sorting algorithm using soft polynomial approximations, which rigorously fails to generalize and confirms that the algorithmic requirements of this task severely stress the natural inductive biases of the architecture.

\end{document}